\documentclass{article}

\usepackage{PRIMEarxiv}

\usepackage[utf8]{inputenc} 
\usepackage[T1]{fontenc}    
\usepackage{hyperref}       
\usepackage{url}            
\usepackage{booktabs}       
\usepackage{amsfonts}       
\usepackage{nicefrac}       
\usepackage{microtype}      
\usepackage{lipsum}
\usepackage{fancyhdr}       
\usepackage{graphicx}       
\graphicspath{{media/}}     
\usepackage{amsfonts}
\usepackage{amsmath}
\usepackage{amsthm}
\usepackage{listings}
\usepackage{color}
\usepackage{multirow}
\usepackage{fancyhdr}
\usepackage{url}
\usepackage[ruled,vlined]{algorithm2e}
\usepackage[UKenglish]{babel}
\usepackage[colorinlistoftodos]{todonotes}

\usepackage{caption}
\usepackage{subcaption}
\newtheorem{thm}{Theorem}

\DeclareMathOperator*{\argmin}{arg\,min}
\DeclareMathOperator*{\argmax}{arg\,max}
\usepackage{algorithmic}
\SetKwInput{KwInput}{Input}                
\SetKwInput{KwOutput}{Output}              

\title{Batch Bayesian Optimization via Particle Gradient Flows}

\author{
  Enrico Crovini \\
  Imperial College London \\
   \And
   Simon L. Cotter \\
   University of Manchester \\
   \And
   Konstantinos Zygalakis \\
   University of Edinburgh \\
   \And
   Andrew B. Duncan \\
  Imperial College London \& Alan Turing Institute
}

\begin{document}
\maketitle

\begin{abstract}
 Bayesian Optimisation (BO) methods seek to find global optima of objective functions which are only available as a black-box or are expensive to evaluate.  Such methods  construct a surrogate model for the objective function, quantifying the uncertainty in that surrogate through Bayesian inference.  Objective evaluations  are sequentially determined by maximising an acquisition function at each step.   However, this ancilliary optimisation problem can be highly non-trivial to solve, due to the non-convexity of the acquisition function, particularly in the case of batch Bayesian optimisation, where multiple points are selected in every step.  In this work we reformulate batch BO as an optimisation problem over the space of probability measures.  We construct a new acquisition function based on multipoint expected improvement which is convex over the space of probability measures.  Practical schemes for solving this `inner' optimisation problem arise naturally as gradient flows of this objective function.  We demonstrate the efficacy of this new method on different benchmark functions and compare with state-of-the-art batch BO methods.
\end{abstract}


\section{Introduction}
Since their inception~\cite{mockus1978application}, Bayesian Optimization (BO) methods  have come to the fore as a widely used class of global optimisation methods for black-box objective functions for which it is  computationally expensive to obtain (possibly noisy) point evaluations~\cite{jones1998efficient, gonzalez2016batch}.  In this context, \textit{black-box} implies that the functional form of the function are not known, including the existence of derivatives.  More specifically, BO seeks to solve the global optimisation problem
\begin{equation}
    \mathbf{x}^* \in \argmin_{ \mathbf{x}} f( \mathbf{x}),
\end{equation}
where $f: \mathbb{R}^d\to \mathbb{R}$ is a function for which only noisy evaluations of the form
\begin{equation}
y_i= f( \mathbf{x}_i) + \varepsilon_i,\quad \varepsilon_i \sim \mathcal{N}(0, \sigma^2),
\end{equation}
are available.  BO methods have been deployed in many areas of science and engineering, ranging from hyper-parameter and architecture tuning for Deep Neural Networks~\cite{victoria2020automatic, alvi2019asynchronous}, to Automatic Chemical Design~\cite{griffiths2020constrained} and A/B testing~\cite{letham2019constrained}.  The development of easy-to-use and efficient software libraries for general purpose Bayesian optimisation, e.g. see \cite{BOsoft, balandat2020botorch},
 have played a substantial role in the widespread use of the methodology.

In sequential BO, given $n$ previous objective evaluations $\mathcal{D}_n = \lbrace (\mathbf{x}_1,y_1), \ldots, (\mathbf{x}_n, y_n)\rbrace$, the next evaluation point $\mathbf{x}_{n+1}$ is selected based on a probabilistic model of the objective function, typically a Gaussian process $\mathcal{GP}(m, k)$ which is conditioned on $\mathcal{D}_n$.  The resulting posterior distribution informs the point selection process through the \emph{acquisition function} which characterises the expected utility of evaluating a particular point for the purposes of optimising the partially observed black-box function $f$.   By identifying candidate evaluation points through the maximization of the acquisition function, and then evaluating the target function at these locations, BO aims to converge to the global optimum efficiently using a limited number of function evaluations.  In practice, the BO algorithm will continue until some stopping condition is met.

Many acquisition functions can be motivated through Bayesian decision theory as being the expected loss associated with evaluating at a point $\mathbf{x}$, and one such acquisition function is  \emph{Expected Improvement} (EI)~\cite{jones1998efficient}. Let $f_n^* = \min_{\mathbf{x} \in \mathcal{D}_n} f(\mathbf{x})$, then EI is defined to be: 
\begin{equation}
    \alpha_{EI}(\mathbf{x}) = \mathbb{E}\left[\mbox{max}(0, f_n^* - f(\mathbf{x}))\, \Big|\, \mathcal{D}_n\right].
\end{equation}
Through its form, EI balances two competing strategies, exploitation, which seeks to encourage points closer to the suspected maximisers of the function, and exploration which seeks to select points with high-uncertainty.   A combination of both strategies is essential to ensure that any global maximiser is attained without getting stuck in local maxima. The exploitation–exploration trade-off is a classic consideration in BO, and the expected improvement
criterion automatically accounts for  both as a result of the Bayesian decision theoretic formulation.  EI has been widely exploited in literature and many sequential methods make use of it, e.g. the Efficient Global Optimization (EGO) algorithm~\cite{jones1998efficient} and Local Penalization \cite{gonzalez2016batch} amongst others, taking advantage of analytical expressions for the EI gradient (e.g. \cite{roustant2012dicekriging}) to reduce the computational cost of the underlying optimisation problem.  Other acquisition functions such as upper confidence bound \cite{auer2002finite} and  probability of improvement \cite{jones2001taxonomy} have also been considered.  See  \cite{garnett_bayesoptbook_2022} for a comprehensive survey. 

In situations where the evaluation of the objective function poses significant computational expense, the use of parallel evaluations offers a workaround to this potential bottleneck.   In this setting, at the $n^{th}$ iteration of a batch BO algorithm, we seek points $\mathbf{x}^1_{n+1}, \ldots, \mathbf{x}^q_{n+1}$ based on the evaluations of the objective function at each point within all the previous batches, i.e $\mathcal{D}_n = \lbrace(\mathbf{x}^j_{i}, y^j_{i}) \, :\, i=1,\ldots, n, j=1,\ldots, q\rbrace$.  Based on the decision-theoretic foundations of EI, a batch-sequential algorithm was derived in \cite{schonlau1997computer,ginsbourger2008multi, ginsbourger2009metamodeles,ginsbourger2010kriging}  making use of a multi-point expected improvement acquisition function, known as $q$-EI, where $q$ indicates the size of the batch.  Choosing batches of points based on this acquisition function enables efficient parallel exploration of the state-space which lends itself particularly well to situations where multiple CPU resources are available.

More specifically, let once again $f_n^* = \min_{\mathbf{x} \in \mathcal{D}_n} f(\mathbf{x})$, we define the multi-point acquisition function by 
\begin{equation}\label{qEI}
    \text{q-EI}(\mathbf{x}_1,\ldots,\mathbf{x}_q) = \mathbb{E} \left[ \max(0, f_n^* - \min_{i = 1,\ldots,q} f(\mathbf{x}_i)) | \mathcal{D}_n \right].
\end{equation}
Since q-EI involves the minimum over multiple dependent random variables, it represents a richer acquisition function, which reduces to EI when $q=1$. We note that it is invariant to permutations. As shown in \cite{ginsbourger2008multi}, q-EI will be higher at  sets of points that each have high EI, but which maintain distance from each other. Therefore, by maximising q-EI, we will obtain a batch of  diverse evaluation points which promote exploration as well as exploitation.  \\

An analytic formula for q-EI is available but its calculation becomes intractable as the dimensionality of the domain $d$ and size of the batch $q$ grow large~\cite{ginsbourger2009two,chevalier2013fast}. For $q \geq 3$ one often resorts to approximating the gradient of the $q$-EI acquisition function using a Monte-Carlo estimator, for example~\cite{wang2020parallel} where MOE-qEI, a batch BO scheme based on an unbiased estimator of the gradient of q-EI  is derived.  Other approaches to batch-BO have been considered, including qKG~\cite{wu2016parallel}, which makes use of the batch version of the knowledge gradient acquisition function \cite{frazier2008knowledge}, d-KG~\cite{wu2017bayesian}, which is the extension of the last method to the case when derivative information for the target function is available and parallel predictive entropy search~\cite{shah2015parallel}.

\subsection{Goal and Contributions of this work}
One aspect of batch-BO setting which has not be directly addressed is the difficulties faced when optimising the acquisition function.   As noted in \cite{wilson2018maximizing},  acquisition functions are routinely non-convex, high-dimensional, and intractable, and this situation is exasperated in the batch-setting, where the complex interactions of evaluation points through the acquisition function give rise to numerous local maxima.   Approaches to performing the "inner-loop"  optimisation step in BO (parallel and sequential) were studied in \cite{wilson2018maximizing}.  Through re-parametrisation they demonstrate how acquisition functions estimated via Monte Carlo integration are consistently amenable to gradient-based optimization.  Secondly, they identify a wide family of acquisition functions which satisfy the sub-modularity relationship, which if satisfied, permit greedy approaches to optimization. \\

In this work, we investigate an alternative {probabilistic} reformulation of batch Bayesian optimisation.  Under this different lens, the inner optimisation loop can be viewed as the optimisation of an \emph{acquisition functional}, defined over the space of probability measures.  In this setting, we construct an acquisition functional, based on $q$-EI, which is provably concave, and for which convergence to the minimum probability measure can be guaranteed through the introduction of a regularisation term which renders the acquisition function strictly concave.   From a batch-BO perspective, this regularisation increases the strength of a repulsing force which promotes diversity of the sample points.  
This probabilistic reformulation immediately exposes a number of natural algorithms for approximating the evaluation points.  Indeed, by taking different gradient flows of the acquisition functional, we recover different algorithms.   We compare and contrast these different approaches  on several synthetic problems and a parameter calibration problem. \\

\section{Probabilistic Reformulation of Batch BO.}
In this section we present the proposed probabilistic reformulation of Batch BO and introduce a natural acquisition functional. Let $\mathcal{D} \subset \mathbb{R}^{d}$ be a compact, simply-connected domain. Let $k:\mathcal{D}\times \mathcal{D}\rightarrow \mathbb{R}$ be a kernel function on $\mathcal{D} \subset \mathbb{R}^d$.  We shall assume that $k$ is continuous on $\mathcal{D}\times \mathcal{D}$, so that the Gaussian process $f \, | \, \mathcal{D}_n$ has almost surely continuous realisations on $\mathcal{D}$.  

In the probabilistic formulation of Batch BO, at the $n^{th}$ step of the scheme, we seek a probability measure $\nu^*$ which maximises some acquisition functional $\mathcal{F}$, i.e. we seek
\begin{equation}
    \nu^* \in \arg\max_{\nu \in \mathcal{P}(\mathcal{D}^q)} \mathcal{F}[\nu],
\end{equation}
where $\mathcal{P}(\mathcal{D}^q)$ is the set of probability measures on $\mathcal{D}^q$, i.e. the Cartesian product of $q$ copies of $\mathcal{D}$.  The functional $\mathcal{F}$ must be constructed to favour probability measures which have high probability in configurations which concentrate around global optimum of the unknown objective function $f$ based on the previously observed noisy evaluations.  As with a standard acquisition function, $\mathcal{F}$ must be chosen to balance between exploration and exploitation.   To this end it makes sense to base $\mathcal{F}$ on existing, well-established acquisition functions.  

As a naive initial attempt, we can define the multi-point expected improvement acquisition functional to be  $\mathcal{F}:\mathcal{P}(\mathcal{D}^q)\rightarrow \mathbb{R}$ to be
\begin{equation}\label{generalqEI}
    \mathcal{F}\left[\nu \right] = \mathbb{E}_{(\mathbf{x}_1,\ldots,\mathbf{x}_q) \sim \nu }\left[ \mathbb{E}_{f} \left[ \max\left(0, f_n^* - \min_{i = 1,\ldots,q} f(\mathbf{x}_i)\right) | \mathcal{D}_n \right]\right].
\end{equation}

The focus is therefore shifted on the joint distribution of the particles' location, not the query points themselves. By maximising this quantity, we would like to retrieve a distribution such that if particles are drawn from it, they achieve the  q-EI optimum.  We note that this is a generalisation of the formulation presented in \cite{gong2019quantile} when $q=1$.

While  \eqref{generalqEI} is a natural probabilistic generalisation of $q$-EI, this is of limited benefit as the resulting scheme will inherit most of the challenges arising when optimising standard $q$-EI. To overcome this, we shall restrict the space of admissible solutions to provide a concave relaxation of the optimisation problem.  To this end, we shall assume that $\nu$ takes the form of a product measure, i.e. $\nu = \mu \times \ldots \times \mu \in \mathcal{P}(\mathcal{D}^q)$, where $\mu \in \mathcal{P}(\mathcal{D})$.  
To simplify notation, we define the functional $F[\mu] := \mathcal{F}[\mu \times \ldots \times \mu]$ on $\mathcal{P}(\mathcal{D})$. 

\begin{thm}\label{Theorem 1} Let $f$ be the Gaussian process with continuous kernel $k$ on $\mathcal{D}$, conditioned on the previous $n$-observations. Then  the functional $\mu \rightarrow F\left[\mu \right]$ is concave on $\mathcal{P}(\mathcal{D})$ for $q > 1$.
\end{thm}

\begin{proof}
We can write the multilinear form as

\begin{equation}
    F[\mu] = \int\ldots\int \text{q-EI}(\mathbf{x}_1, \dots, \mathbf{x}_q)\mu(d\mathbf{x}_1)\ldots \mu(d\mathbf{x}_q),
\end{equation}

Let $\lambda \in [0,1]$ and let $\mu_0$, $\mu_1$ be two probability measures with support on a compact set $ A \subset \mathbb{R}^d$.  Define the interpolated probability measure $\mu_{\lambda} = \lambda \mu_1 + (1-\lambda)\mu_0$.  We present a proof of concavity for general $q>1$.

Consider \begin{equation}F[\mu_\lambda] = \int \ldots \int \text{q-EI}(\mathbf{x}_1,\ldots,\mathbf{x}_q)\mu_{\lambda}(d\mathbf{x}_1)\ldots\mu_{\lambda}(d\mathbf{x}_q).\end{equation}   

We can rewrite this functional as 
\begin{align*}
    F[\mu_\lambda] =&\int \ldots \int  \mathbb{E}_{f}\left[\max \left(0, f_n^* - \min_{i=1,\ldots,q} f(\mathbf{x}_i)\right)\right]\mu_\lambda(d\mathbf{x}_1)\ldots\mu_\lambda(d\mathbf{x}_q),\\
    =& \int \ldots \int  \mathbb{E}_{f}\left[f_n^* - \min_{i=1,\ldots,q} \widetilde{f}(\mathbf{x}_i)\right]\mu_\lambda(d\mathbf{x}_1)\ldots\mu_\lambda(d\mathbf{x}_q),
\end{align*}
where $\widetilde{f}(\mathbf{x})=\min(f(\mathbf{x}),f_n^*)$.  We can express the integrand as an integral.  Indeed, for all $\mathbf{x}\in \mathcal{D}$ and for every realisation of the GP $f$, where $\tilde{L} = \min_\mathbf{x} \widetilde{f}(\mathbf{x})$ (which exists as $f$ has a compact domain), we have:
\begin{align*}
 f_n^* - \min_{i = 1,\ldots,q} \widetilde{f}(\mathbf{x}_i)  = & (f_n^* - \tilde{L}) - \left ( \min_{i = 1,\ldots,q} \widetilde{f}(\mathbf{x}_i) - \tilde{L} \right ), \\
=&
\int_0^\infty \mathbf{1}[s<f_n^* - \tilde{L}] - \mathbf{1}\left[s < \min_{i=1,\ldots,q} \widetilde{f}(\mathbf{x}_i) - \tilde{L} \right] \,ds, \\
=& \int_0^\infty \mathbf{1}[s<f_n^*- \tilde{L}]^q - \prod_{i=1}^q\mathbf{1}\left[s < \widetilde{f}(\mathbf{x}_i)  - \tilde{L}\right] \,ds.
\end{align*}

Through repeated use of Fubini-Tonelli, we can swap the order of the integrals to obtain
\begin{align*}
F[\mu_{\lambda}] &= \int_0^\infty \mathbb{E}_{f_n}\left[\int \ldots \int \mathbf{1}[s<f_n^*- \tilde{L}]^q \mu_{\lambda}(d\mathbf{x}_1)\ldots\mu_{\lambda}(d\mathbf{x}_q)\right]\,ds \\
    &-\int_0^\infty \mathbb{E}_{f_n}\left[\int\ldots\int \prod_{i=1}^q\mathbf{1}\left[s < \widetilde{f}(\mathbf{x}_i)  - \tilde{L}\right] \mu_{\lambda}(d\mathbf{x}_1)\ldots\mu_{\lambda}(d\mathbf{x}_q)\right]\,ds.
\end{align*}
We now use the fact that $\mathbf{x}_i\sim \mu_{\lambda}$ are i.i.d. random variables (up to sets of measure zero, due to the conditioning of the product measure), and so we can considerably simplify the above integrals as 
\begin{equation}
\label{eq:differences}
    F[\mu_{\lambda}] = \int_0^\infty \mathbb{E}_{f_n}\left[\left(\int\mathbf{1}[s<f_n^*- \tilde{L}]\mu_{\lambda}(d\mathbf{x})\right)^q - \left(\int \mathbf{1}[s<\widetilde{f}(\mathbf{x})- \tilde{L}]\mu_{\lambda}(d\mathbf{x})\right)^q\right]\,ds.
\end{equation}

Concavity of the functional is equivalent to the condition that
\begin{equation}
\label{eq:concavity}
    F[\mu_{\lambda}] \geq \lambda F[\mu_1] + (1-\lambda)F[\mu_0]. 
\end{equation}

First note that, for all $\lambda \in [0,1]$,
\begin{equation}
\int\mathbf{1}[s<f_n^*- \tilde{L}]\mu_{\lambda}(d\mathbf{x})=\mathbf{1}[s<f_n^*- \tilde{L}], 
\end{equation}
since $\mu_0$ and $\mu_1$ are probability measures, and the integrand is independent of $\mathbf{x}$.  The $f_n^*- \tilde{L}$ terms cancel across the concavity condition and so  condition \eqref{eq:concavity}  reduces to:
\begin{align*}
    \int_0^\infty \mathbb{E}_{f_n}\left[\left(\int \mathbf{1}[s<\widetilde{f}(\mathbf{x})- \tilde{L}]\mu_{\lambda}(d\mathbf{x})\right)^q\right]   \leq& \lambda \int_0^\infty \mathbb{E}_{f_n}\left[\left(\int \mathbf{1}[s<\widetilde{f}(\mathbf{x})-\tilde{L}]\mu_{1}(d\mathbf{x})\right)^q\right] \\ &  + (1-\lambda)\int_0^\infty \mathbb{E}_{f_n}\left[\left(\int \mathbf{1}[s<\widetilde{f}(\mathbf{x})-\tilde{L}]\mu_{0}(d\mathbf{x})\right)^q\right].
\end{align*}
Note the change in signs.  This is true if, for each realisation $f$ and each $s\geq 0$ that
\begin{align*}
    \left(\int \mathbf{1}[s<\widetilde{f}(\mathbf{x})-\tilde{L}]\mu_{\lambda}(d\mathbf{x})\right)^q  =& \left(\lambda \int  \mathbf{1}[s<\widetilde{f}(\mathbf{x})-\tilde{L}]\mu_{1}(d\mathbf{x})  +(1-\lambda) \int \mathbf{1}[s<\widetilde{f}(\mathbf{x})-\tilde{L}]\mu_{0}(d\mathbf{x})\right)^q, \\
    \leq& \lambda  \left(\int \mathbf{1}[s<\widetilde{f}(\mathbf{x})-\tilde{L}]\mu_{1}(d\mathbf{x})\right)^q  +(1-\lambda)  \left(\int \mathbf{1}[s<\widetilde{f}(\mathbf{x})-\tilde{L}]\mu_{0}(d\mathbf{x})\right)^q.
\end{align*}
For fixed $f, s$, let 
\begin{align*}
A_i(s,f) &= \int \mathbf{1}[s < \widetilde{f}(\mathbf{x})-\tilde{L}]\mu_i(d\mathbf{x}) = \mathbb{P}[s < \widetilde{f}(\mathbf{X}_i)-\tilde{L}\,|\, f] \geq 0,
\end{align*}
where $\mathbf{X}_i \sim \mu_i$.
The above inequality can be written as 
\begin{equation}
\left(\lambda A_1 + (1-\lambda)A_0\right)^q \leq \lambda A_1^q + (1-\lambda)A_0^q.
\end{equation}
Since $q(q-1)u^{q-2} \geq 0$ for $q> 1$ and $u>0$ it follows that the  function $u\rightarrow u^q$ is convex on $\mathbb{R}^+$, so this inequality must hold.  Since the integrals with respect to $f$ and $s$ are linear this implies concavity of the functional $F[\mu]$.
\end{proof}

While the above result ensures the concavity of $F$, it is clearly not strictly concave.  To ensure the existence of a unique, well-defined minimiser, we perturb the acquisition functional with a regularisation term over $\mathcal{P}(\mathcal{D})$ to guarantee strict convexity. Specifically, we regularise by introducing a negative log-entropy term of the form
\begin{equation}
\mathrm{Reg}[\mu] = - \int \log \mu(x) \mu(x) dx,
\end{equation}
so that the inner optimisation problem becomes:
\begin{equation}\label{Objective}
    \mbox{Find } \mu_{\alpha}^* \in \argmax_\mu \{ L_{\alpha}[\mu]:= F[\mu] + \alpha \mathrm{Reg}[\mu] \} ,
\end{equation}
where $\alpha > 0$.   It is straightforward to show that $L_{\alpha}$ is strictly concave for any $\alpha > 0$. Intuitively, the effect of the regularisation will be to promote distributions which have a wider spread.  Note that $q$-EI  already rewards distributions of particles which are promoting diversity, so that the regularisation term merely amplifies that effect further.  
.


\section{Gradient Flows on the Space of Probability Measures}

To solve the optimisation \eqref{Objective} we will introduce a gradient flow over the space of probability measures.   In this context, \emph{gradient flows} represent a set of methods that define strategies to propagate an initial density $\rho_0$ so that it will converge asymptotically to the distribution of interest.  Intuitively, given an initial candidate distribution $\mu_0$, the gradient flow will define a sequence of probability measures $\mu_1, \mu_2, \ldots$, such that $\mu_i \rightarrow \mu_{\alpha}^*$ and where the trajectory follows a path of steepest ascent of $L_{\alpha}$ in an appropriately chosen geometry on $\mathcal{P}(\mathcal{D})$.  Different choices of the underlying geometry will give rise to different gradient flows, which will in turn result in different algorithms for solving \eqref{Objective}


Various gradient flows have been applied in different problems for example, for variational inference~\cite{rezende2015variational}, clustering~\cite{agnelli2010clustering}, variance reduction in MCMC~\cite{albergo2019flow},  implicit generative model training using Maximum Mean Discrepancy~\cite{arbel2019maximum}, and using a sliced Wasserstein metric~\cite{liutkus2019sliced}. In addition, they have been deployed in Reinforcement learning for policy optimisation~\cite{zhang2018policy} and crowd-behaviour modelling~\cite{maury2010macroscopic} among many other applications.

In the following sections, we will derive two specific gradient flows for the objective function \eqref{Objective}, one based on the Stein geometry \cite{duncan2019geometry}, and another using the Wasserstein geometry \cite{santambrogio2017euclidean}.

\subsection{Stein Gradient Flow}
\label{sec:stein}

Stein Variational Gradient Descent (SVGD) is a sampling algorithm that relies on deterministic updates of a set of particles, which are propagated so that the KL divergence between their empirical distribution and the target distribution decreases optimally \cite{liu2016stein}. It was later noticed in \cite{liu2017stein} that the mean-field limit of the SVGD update step could be interpreted as an equivalent gradient flow on the space of probability densities in the \emph{Stein geometry}, whose convergence and properties were investigated in \cite{duncan2019geometry}.

To construct the Stein Gradient flow, we maximise \eqref{Objective}, over a sequence of probability densities $\mu_1, \mu_2, \ldots$, where $\mu_i$ is the push-forward of $\mu_{i-1}$ under a sequence of maps $T_i:\mathcal{D}\rightarrow\mathcal{D}$ of the form $T_i(x) = x + \epsilon \Phi_i(x)$.  The vector fields $\Phi_i$ are chosen to advect the measure $\mu_{i_1}$ in the direction of steepest ascent of $L_{\alpha}$.  More specifically, suppose that $\mu$ is a density on $\mathbb{D}$.  Then, to calculate the next probability density in the gradient flow,  we must solve the problem
\begin{equation}\label{ObjectivePhi}
    \Phi^* = \argmax_{\Phi \in \mathcal{C}} \frac{d}{d\epsilon} \left[ F[\mathbf{T}_\#(\mu)]+\alpha \mathrm{Reg}[\mathbf{T}_\#(\mu)] \right]  \Bigr \rvert_{\epsilon = 0},
\end{equation} 
where $\mathcal{C}$ is an admissible class of vector fields on $\mathcal{D}$.    We can focus on the two terms in \eqref{ObjectivePhi} separately. Through easy calculations we obtain that the first term becomes:
\begin{equation}\label{ObjectiveF}
\frac{d}{d\epsilon} F\left[\mathbf{T}_\#(\mu)\right]\Bigr \rvert_{\epsilon = 0} = \int \ldots \int \sum_{i=1}^q \frac{\partial \text{q-EI}}{\partial \mathbf{x}_i }(\mathbf{x}_1,\ldots,\mathbf{x}_q)\Phi(\mathbf{x}_i) \mu(d\mathbf{x}_1)\ldots\mu(d\mathbf{x}_q).
\end{equation}

Using Jacobi's formula, the regularisation term becomes then:
\begin{equation}\label{ObjectiveReg}
\frac{d}{d\epsilon}\mathrm{Reg}\left[\mathbf{T}_\#(\mu)\right]\Bigr \rvert_{\epsilon = 0} = \int \nabla \cdot \Phi(\mathbf{x}) \mu(d\mathbf{x}).
\end{equation}

For full calculations and details please refer to Appendix \ref{sec:AppA}.  If $\mathcal{C}$ is chosen to be $\{ f \in \mathcal{H}^d: ||f|| \leq 1\}$ where  $\mathcal{H}^k$ is the product of $d$ copies of a Reproducing Kernel Hilbert Space (RKHS)  $\mathcal{H}$ with kernel differentiable $k(\cdot, \cdot)$, then we obtain that:
\begin{equation}\label{Phi}
    \Phi^*(\mathbf{z}) =  \underbrace{\int \ldots \int \sum_{i=1}^q \frac{\partial \text{q-EI}}{\partial \mathbf{x}_i }(\mathbf{x}_1,\ldots,\mathbf{x}_q)k(\mathbf{x}_i,\mathbf{z}) \mu(d\mathbf{x}_1)\ldots\mu(d\mathbf{x}_q)}_{\Phi_1^*(\mathbf{z})} + \underbrace{\alpha \int \nabla_\mathbf{x} k(\mathbf{x}, \mathbf{z}) \mu(d\mathbf{x})}_{\Phi_2^*(\mathbf{z})}.
\end{equation}

In order to calculate \eqref{Phi}, we must be able to evaluate $\frac{\partial \text{q-EI}}{\partial \mathbf{x}_i }$. However, as q-EI is not differentiable, we will now focus on carefully redefining the acquisition function, so that its partial derivatives will exist. To address this issue, we shall introduce the smoothed approximation  $\text{q-EI} \approx g(\alpha)$, where $g$ is the LogSumExp function with the first argument set to 0, in place of $\max(0, \cdot)$, and $\alpha(\mathbf{x}_1,\ldots,\mathbf{x}_q) = (f^* - f(\boldsymbol{x}_1),\ldots,f^* - f(\boldsymbol{x}_q))^T$. In this way, we have constructed an estimator of \eqref{qEI} such that the gradients are easy to evaluate and we can write, for $ i = 1,\ldots,q$: 
\begin{equation} \label{gradient}
\frac{\partial g(\alpha)}{\partial \mathbf{x}_i }(\mathbf{x}_1,\ldots,\mathbf{x}_q) = - \frac{e^{(f^*-f(\mathbf{x}_i))}}{ 1 + \sum_{k=1}^{q} e^{(f^*-f(\mathbf{x}_k))}} \nabla f ({\mathbf{x}_i}).
\end{equation} 

Hence, note that given any permutation $\sigma \in S(q)$ we have:
\begin{align}\label{permutation}
    \frac{\partial g(\alpha)}{\partial \mathbf{x}_i} (\mathbf{x}_{\sigma(1)},\ldots,\mathbf{x}_{\sigma(q)}) &= - \frac{e^{(f^*-f(\mathbf{x}_{\sigma(i)}))}}{1+ \sum_{k=1}^{q} e^{(f^*-f(\mathbf{x}_{\sigma(k)}))}} \nabla f ({\mathbf{x}_{\sigma(i)}}), \nonumber \\
    &= \frac{\partial {g(\alpha)}}{\partial \mathbf{x}_{\sigma(i)}}(\mathbf{x}_{1},\ldots,\mathbf{x}_{q}).
\end{align}

To estimate the integral in \eqref{Phi}, one could resort to standard Monte-Carlo estimates of this $q$-dimensional integral.  Noting that $h(\mathbf{x}_1,\ldots,\mathbf{x}_q) = \sum_{i=1}^q \frac{\partial g(\alpha)}{\partial \mathbf{x}_i }(\mathbf{x}_1,\ldots,\mathbf{x}_q)k(\mathbf{x}_i,\mathbf{z})$ is symmetric in its arguments,  we can replace this with a U-statistic \cite{hoeffding1992class} estimator of the form,
\begin{equation}\label{MCMCestimator}
\widetilde{\Phi}_1^*(\mathbf{z}) = \frac{1}{\binom{N}{q}} \sum_{\gamma \in C_{N,q}} \sum_{m=1}^q \frac{\partial g(\alpha)}{\partial \mathbf{x}_m }(\mathbf{x}_{\gamma(1)},\ldots,\mathbf{x}_{\gamma(q)})k(\mathbf{x}_{\gamma(m)},\mathbf{z}),
\end{equation}
where $C_{N,q}$ indicates all possible combinations of $q$ elements over the $N$ available.  For $N \gg q$ this provides an accurate approximation of the $q$-fold integral requiring only $\binom{N}{q}$ gradient evaluations instead of a naive Monte-Carlo estimator which would require $N^q$. These types of estimators have favourable properties: firstly, if we let $ \sigma_1^2 = \mathrm{Var}[ \mathbb{E}[h(\mathbf{X}_1,...,\mathbf{X}_q) | \mathbf{X}_1] ]$, then, as a corollary to what showed in \cite{hoeffding1992class}, we know that the variance of the U-statistics is asymptotically $(q^2 \sigma_1^2)/N$ for increasing $N$ \cite{wang2014variance}; secondly, an equivalent strong law of large numbers holds, meaning that if the integral $\Phi_1^*$ as given in \eqref{Phi} is finite, then the estimator as given by  $\widetilde{\Phi}_1^*$ \eqref{MCMCestimator} converges almost surely as $N$ grows to infinity \cite{hoeffding1961strong}.

The second integral in  \eqref{Phi} can be evaluated using a standard Monte Carlo approximation. Given $\mathbf{x}_1, \ldots, \mathbf{x}_N \sim \mu(d\mathbf{x})$ i.i.d. then we use the estimator
\begin{equation}\label{MCMC Pen}
\widetilde{\Phi}_2^*(\mathbf{z}) = \frac{\alpha}{N} \sum_{i = 1}^N \nabla_{\mathbf{x}_i} k(\mathbf{x}_i, \mathbf{z}).
\end{equation}
Combining these terms we obtain
\begin{equation}\label{final Phi}
\widetilde{\Phi}^*(\mathbf{z}) =\frac{1}{\binom{N}{q}} \sum_{\gamma \in C_{N,q}} \sum_{i=1}^q
   \mathbb{E}\left[\frac{\partial g(\alpha)}{\partial \mathbf{x}_i }(\mathbf{x}_{\gamma(1)},\ldots,\mathbf{x}_{\gamma(q)})\right]k(\mathbf{x}_{\gamma(i)},\mathbf{z}) +  \frac{\alpha}{N} \sum_{i = 1}^N \nabla_{\mathbf{x}_i} k(\mathbf{x}_i, \mathbf{z}),
\end{equation}
where the expectation can be approximated as we will show in Section \ref{GP estimation}.

\subsection{Wasserstein Particle Gradient Flow}
\label{sec:wasserstein}
In this section we present an alternate gradient flow for \eqref{Objective}, based on the Wasserstein geometry on the set $\mathcal{P}_d(\mathcal{D})$ of densities on $\mathcal{D}$.  This is the pseudo-Riemmanian metric induced by the $\mathcal{W}_2$ distance. Given that probability measures with atoms are not admissible for the regularised objective function $L_{\alpha}$, this geometry is a natural candidate for formulating a gradient flow.   In this setting,  the gradient flow will take the form of an evolution of densities satisfying 
\begin{equation}
\label{wass:gf}
    \partial \mu_t  = \nabla_{\mathcal{W}_2} L_{\alpha}(\mu_t),
\end{equation}
where $\nabla_{\mathcal{W}_2}$ indicates the notion of gradient in the metric space we are in~\cite{lavenant2018dynamical}. These kind of flows are strongly connected to partial differential equations~\cite{santambrogio2017euclidean}, and similarly to gradient descent, they move along the steepest direction to find the optimal measure of interest. \\

As seen in \cite{mokrov2021large}, given the functional $L_{\alpha}$ in \eqref{Objective}, the Wasserstein gradient flow \eqref{wass:gf} can be expressed as a continuity equation of the form:
\begin{equation}\label{Wasserstein GF}
    \frac{\partial \mu_t}{\partial t} = \mathrm{div}(\mu_t \nabla_\mathbf{x} L_{\alpha}'[\mu_t]),
\end{equation}
where $L_{\alpha}'[\mu_t]$ is the first variation of $L$. To calculate this, we again consider separately the objective and regularisation terms in $L_{\alpha}$. For the objective term  we can directly state:
\begin{equation}
    \frac{d}{d\epsilon} F[\mu + \epsilon \nu]  \Bigr \rvert_{\epsilon = 0} = \int \left(q \int \ldots \int \text{q-EI}(\mathbf{x}_1, \ldots, \mathbf{x}_q)\mu(d\mathbf{x}_1) \ldots \mu(d\mathbf{x}_{q-1}) \right) \nu(d\mathbf{x}_q).
\end{equation}
We obtain that the first variation for the functional $F$ is
\begin{equation}
F'[\mu_t] = H[\mu, \mathbf{x}]= q \frac{\partial}{\partial{\mathbf{x}_q}} \int \ldots \int  \left(\text{q-EI}(\mathbf{x}_1, \ldots,\mathbf{x}_{q-1}, \mathbf{x}) \right) \mu(d\mathbf{x}_1)\ldots\mu(\mathbf{x}_{q-1}).
\end{equation}

We can conclude by writing: 
\begin{equation}\label{FV-OBJ}
     \mathrm{div}(\mu_t \nabla_{\mathbf{x}}F'[\mu_t]) = \partial_{\mathbf{x}}(\mu_t) H[\mu_t, \mathbf{x}] + \mu_t \partial_{\mathbf{x}}H[\mu_t, \mathbf{x}].
\end{equation}

Full details of the derivations for both terms are given in Appendix B.   On the other hand we have that:
\begin{align*}
    \frac{d}{d\epsilon} \text{Reg}[\mu + \epsilon \nu]  \Bigr \rvert_{\epsilon = 0} = \int \log (\mu(\mathbf{x})) + 1 )\nu(d\mathbf{x}),
\end{align*}
from which we derive that the first variation for the regularisation term is Reg$'[\mu] = \log (\mu) + 1 $. Hence, we can write:
\begin{align}\label{FV-REG}
    \mathrm{div}(\mu_t \nabla \mathrm{Reg}'[\mu])  = \Delta \mu_t.
\end{align}
Therefore, by summing together \eqref{FV-OBJ} and \eqref{FV-REG} we obtain:
\begin{equation}\label{Flow}
    \frac{\partial \mu_t}{\partial t} = \partial_{\mathbf{x}}(\mu_t) H[\mu_t, \mathbf{x}] + \mu_t \partial_{\mathbf{x}}H[\mu_t, \mathbf{x}] + \Delta \mu_t.
\end{equation}

As highlighted in \cite{liutkus2019sliced, durmus2022sticky}, this  PDE is a nonlinear Fokker-Planck equation which characterises the evolution of density of a McKean-Vlasov type  stochastic differential equation of the form:
\begin{equation}
\label{eq:sde}
d\mathbf{X}_t = H[\mu_t, \mathbf{X}_t]dt + \sqrt{2}dW_t, \quad \mu_t = \mathrm{Law}(\mathbf{X}_t),
\end{equation}
where $W_t$ is a standard Wiener process on $\mathbb{R}^d$. If we consider  a set of moving particles $\{{\mathbf{x}}_t^{(j)}\}_{j=1}^N$ and their associated empirical distribution $\hat{\mu}_t = \frac{1}{N}\sum_{i= 1}^N \delta_{{\mathbf{x}}_t^{(j)}}$, we can plug in these quantities to obtain a system of $n$ particles approximating \eqref{eq:sde}:
\begin{equation}
\label{eq:sde_particle}
    d{\mathbf{x}}^{(i)}_t = H[\hat{\mu}_t, {\mathbf{x}}^{(i)}_t]dt + \sqrt{2}dW^{(i)}_t, \quad i = 1,\ldots,N.
\end{equation}

Establishing the convergence of \eqref{eq:sde_particle}  to its mean field limit \eqref{eq:sde} as $N\rightarrow \infty$ is a challenging problem which we defer to future work. This set of n non-linear stochastic differential equations can be further approximated by introducing an Euler-Maruyama discretisation as follows, for $i = 1, \ldots, N$:
        \begin{equation}\label{SDE_PGF}
        \mathbf{x}^{(i)}_{t+1} = \mathbf{x}^{(i)}_{t} + \Delta t H[\hat{\mu}_t, {\mathbf{x}}^{(i)}_t] + \sqrt{2\Delta t \alpha}Z^i_t, \quad t=1,\ldots,T,
    \end{equation}
where $Z^{(i)}_t \sim \mathcal{N}(\mathbf{0}, \mathbf{I}_{d\times d})$ is a standard multivariate normal vector, after choosing an initial set of particles $\{\mathbf{x}^{(i)}_{0}\}_{i=1}^N$.

Computing the drift term poses a significant challenge due to the high-dimensional integral.  Using the smooth approximation $g$ to the non-differentiable integrand, as defined in the previous section, then we have that
\begin{equation}\label{H}
H[\mu , \mathbf{x}] = q \mathbb{E}_{\mathbf{X}_1,\ldots,\mathbf{X}_{q-1} \sim \mu} \left[ \frac{\partial g(\alpha)}{\partial \mathbf{x}_q} (\mathbf{X}_1,\ldots,\mathbf{X}_{q-1}, \mathbf{x}) \right].
\end{equation}

We note that the function $f_{\mathbf{x}}(\mathbf{x}_1,\ldots,\mathbf{x}_{q-1}) = \frac{\partial g(\alpha)}{\partial \mathbf{x}_q} (\mathbf{x}_1,\ldots,\mathbf{x}_{q-1}, \mathbf{x})$, of which we want to take the expectation, is symmetric in its arguments. Therefore we can once use again the U-statistic. Consider the set of particles introduced in previous section $\{{\mathbf{x}}_t^{(i)}\}_{i = 1}^N$ and their empirical distribution $\hat{\mu}_t$  and write:
\begin{equation}\label{H-Ustats}
    \tilde{H}[\hat{\mu}_t, \mathbf{x}] = \frac{q}{\binom{n}{q-1}} \sum_{\gamma \in C_{N,q-1}} \frac{\partial g(\alpha)}{\partial \mathbf{x}_q }({\mathbf{x}}_t^{\gamma(1)},\ldots,{\mathbf{x}}_t^{\gamma(q-1)}, \mathbf{x}).
\end{equation}where $ C_{N,q-1}$ indicates the possible combinations of $q-1$ objects over $N$ available.  Given an initial set of points $\{\mathbf{x}^{(i)}_{0}\}_{i=1}^N$ and for $i = 1,\ldots, N$, the resulting scheme is given by
\begin{equation}\label{EM - step}
    \mathbf{x}^{(i)}_{t+1} = \mathbf{x}^{(i)}_{t} + \frac{\epsilon_w}{{\binom{N}{q-1}}} \sum_{\gamma \in C_{N,q-1}}  \mathbb{E} \left[\frac{\partial g(\alpha)}{\partial \mathbf{x}_q }({\mathbf{x}}_t^{\gamma(1)},\ldots,{\mathbf{x}}_t^{\gamma(q-1)},\mathbf{x}^{(i)}_t)\right] + \alpha_w Z^i_t, \quad t=1,\ldots,T,
\end{equation}
where $\epsilon_{w}:= \Delta t q$, $\alpha_w:=  \sqrt{2\Delta t \alpha}$, and $Z^i_t \sim \mathcal{N}(\mathbf{0}, \mathbf{I}_{d\times d})$ are i.i.d random vectors. The expectation is estimated as explained in Section \ref{GP estimation}.

\subsection{Estimation of q-EI gradient}\label{GP estimation}
In order to be able to evaluate \eqref{MCMCestimator} and \eqref{H-Ustats} we would need to be able to have function evaluations and gradient information, as given by Equation \eqref{gradient}. Since $f$ is a GP, we shall instead estimate $\mathbb{E}[\frac{\partial g(\alpha)}{\partial \mathbf{x}_i }(\boldsymbol{x})]$, $i = 1,\ldots,q$ and use this quantity instead. This is easily achievable thanks to the properties of GPs; as $f$ is modelled using a GP we know that its gradient is still a GP~\cite{rasmussen2003gaussian}. Moreover, the joint distribution $[f, \nabla f]$ will still be a GP from which we can easily sample from. More formally, if we assume that the prior mean is 0 and covariance kernel $k$ we will have the following:

\begin{equation}
\left[\begin{array}{c}
f \\
\nabla f \\
\end{array}\right] \sim \mathcal{G} \mathcal{P}\left(\mathbf{0}, \mathbf{K}_{\left[\mathbf{f}, \nabla \mathbf{f}\right]}\right),
\end{equation}
where
\begin{equation}\quad 
\mathbf{K}_{\left[{f}, \nabla {f}\right]} = 
\left[\begin{array}{cc}
\mathbf{k} & \mathbf{k}_{[f, \nabla f]}  \\
\mathbf{k}_{[\nabla f, f]} & \mathbf{k}_{[\nabla f, \nabla f]} 
\end{array}
\right].
\end{equation}

If $k$ is twice differentiable, the element of the matrix can be written as follows:

\begin{equation}\mathbf{k}_{[f, \nabla f]}(x,y) = \frac{\partial}{\partial y}k(x,y),\end{equation}
\begin{equation}\mathbf{k}_{[\nabla f, \nabla f]} = \frac{\partial^2}{\partial x\partial y}k(x,y). \end{equation}

Hence, by sampling jointly from the posterior distribution $[f, \nabla f | \mathbb{D}_n]$ we can estimate using Monte Carlo integration the expectation of the gradient in \eqref{gradient}. More formally, call $(f(\mathbf{x}_1),\ldots,f(\mathbf{x}_n)) = \mathbf{f}_{1:n}$. Select $n^*$ new locations $\mathbf{x}^*$ and let us denote $\mathbf{f}^* = \mathbf{f}(\mathbf{x}^*)$ and $\nabla \mathbf{f}^* = \nabla \mathbf{f}(\mathbf{x}^*) $. The vector $[\mathbf{f}_{1:n}, \mathbf{f}^*, \nabla \mathbf{f}^* ] \in \mathbb{R}^{n+n^* \times (d+1)}$, is multivariate normal because of the GP prior we imposed. In other words:

\begin{equation}
\left[\begin{array}{c}
\mathbf{f}_{1: n} \\
\mathbf{f}^{*} \\ 
\nabla \mathbf{f}^* \\
\end{array}\right] \sim \mathcal{N}\left(\mathbf{0},
\left[\begin{array}{ccc}
\mathbf{k}(\mathbf{x}, \mathbf{x}) & \mathbf{k}(\mathbf{x}, \mathbf{x}^*) & \mathbf{k}_{[f, \nabla  f]}(\mathbf{x}, \mathbf{x}^*) \\
\mathbf{k}(\mathbf{x^*}, \mathbf{x}) & \mathbf{k}(\mathbf{x^*}, \mathbf{x^*}) &\mathbf{k}_{[f, \nabla f]}(\mathbf{x^*}, \mathbf{x^*}) \\
\mathbf{k}_{[\nabla f, f]}(\mathbf{x^*}, \mathbf{x})& \mathbf{k}_{[\nabla f, f]}(\mathbf{x^*}, \mathbf{x^*}) & \mathbf{k}_{[\nabla f, \nabla f]}(\mathbf{x^*}, \mathbf{x^*})
\end{array}\right]\right).
\end{equation}

By standard properties of the multivariate normal distribution we can explicitly write down the joint posterior distribution:
\begin{equation}\label{posterior gauss}
    \mathbf{f}^{*} , \nabla \mathbf{f}^{*} \mid \mathbf{f}_{1: n}, \sim \mathcal{N}\left({\mu}\left(\mathbf{x}^{*}\right), {\sigma}^{2}\left(\mathbf{x}^{*}\right)\right).
\end{equation}

Let use denote with $
{\mathbf{k}}^{*}= [\mathbf{k}(\mathbf{x}, \mathbf{x}^*) \quad  \mathbf{k}_{[f, \nabla  f]}(\mathbf{x}, \mathbf{x}^*)]^{\text{T}} 
$. Then we have:
\begin{equation}
\begin{aligned}
{\mu}\left(\mathbf{x}^{*}\right) &= {\mathbf{k}^{*}}^\top \mathbf{k}^{-1}(\boldsymbol{x}, \boldsymbol{x})^\top  \mathbf{f},\\ 
\sigma^2 \left(\mathbf{x}^{*}\right) &= \mathbf{K}_{ \nabla \left[{f}, \nabla {f}\right]}(\boldsymbol{x}^*, \boldsymbol{x}^*) - {\mathbf{k}^{*}}^\top \mathbf{k}^{-1}(\boldsymbol{x}, \boldsymbol{x})\mathbf{k}^{*}.
\end{aligned}
\end{equation}

Hence, we can draw samples from this multivariate normal distribution. Fix $M \in \mathbb{N}$ and select a set of $q$ locations $\boldsymbol{x}^*$. Call $\{\mathbf{f}_j^{*} , \nabla \mathbf{f}_j^{*} \}_{j = 1}^{M}$ the posterior sample evaluated at the desired $q$ inputs. Then we can write, for $i=1,\ldots,q$:

\begin{equation}\label{expected gradient}
    \mathbb{E}\left[\frac{\partial g(\alpha)}{\partial \mathbf{x}_i }(\mathbf{x})\right] \approx -\frac{1}{M} \sum_{j=1}^M \frac{e^{\left(f^{*}-f_j\left(\mathbf{x}_{i}\right)\right)}}{\sum_{k=1}^{q} e^{\left(f^{*}-f_j\left(\mathbf{x}_{k}\right)\right)}} \nabla f_j\left(\mathbf{x}_{i}\right). 
\end{equation}

\section{Implementation}
In the following section we will explain the implementation of both methods. Code\footnote{Available at \url{https://github.com/enricocrovini/BBO-via-PGF}} is written in python and we make use of the \url{jax} library~\cite{jax2018github}. The main steps for the BO scheme based on Stein Gradient Flow are listed in Algorithm \eqref{algo stein}, while the scheme backed on the Wasserstein Gradient Flow is described in Algorithm \ref{algowass}. 


As seen in \ref{GP estimation}, the  GP prior we use to model the target function needs to have a covariance kernel which is at least doubly differentiable; a suitable candidate could be Matern $5/2$ kernel as it fulfils the required property. Additionally, this choice of kernel is motivated as it is not desireable to impose an excessively smooth structure on the target, which, as it is to be considered black-box, might be noisy and not satisfy any regularity assumptions. In this context, to allow more flexibility for the modelling of the target, we allow the possibility to add to the Matern $5/2$ a noise kernel. At each iteration, as the evaluations of the expensive target are retrieved, we optimise the value of the kernel hyperparameters using \url{scikit-learn}. 

Once we have all the tools to sample from the posterior of the GP, conditional on the current information about the function, the inner optimisation loop of q-EI starts. 
The initial particles' locations are initialised using Latin Hypercube Sampling~\cite{mckay2000comparison} using the \url{scikit-optimize} package. These particles' positions give rise to an empirical distribution over the domain in which the target is defined, which approximates the initial density used for both Stein and Wasserstein flows.

The numerical schemes arising from the gradient flows derived in Sections \ref{sec:stein}  and \ref{sec:wasserstein}  are then integrated within a Batch Bayesian Optimisation scheme,  where the unknown objective function $f$ is explored by performing $N$ evaluations at each outer loop step, and with an acquisition functional based on $q$-EI within the inner-loop, as given by \eqref{generalqEI}.  {Here, the batch size $q$ is assumed to be smaller than the ensemble size $N$.}  This offers an alternative strategy to existing Batch BO methods where multiple independent batches of $q$ points are evaluated at every outer loop step and the best-performing batch of $q$ points is selected for evaluation~\cite{wang2020parallel}.   In our approach, the $N$ particles will form multiple batches of size $q$, but information is shared across these multiple batches, and the regularisation term promotes diversity in their exploration of the domain.  
\\

\begin{algorithm}
\caption{Batch BO via Stein PGF}\label{algo stein}
\KwInput{Dataset $\mathbb{D}_{n} = \{ \mathbf{x}_i, f(x_i)\}_{i = 1}^{_{n}}$, { batch size $q$, ensemble size $N$,}
length of optimisation $T$, number of BO iterations $I$, number of GP samples $M$, regularisation constant $\alpha$, Stein Kernel $k$, step size $\epsilon$.}
\KwOutput{Updated dataset $\mathbb{D}_{n}$ with locations at which function has been evaluated.}
\begin{algorithmic}[1]
\FOR{$i= 1,\dots,I$}
    \STATE Fit a GP to $\mathbb{D}_n$; 
        \STATE select random initial location $\mathbf{x}^0 \in \mathbb{R}^{N\times d}$;
         \FOR{$t=0,\dots,T-1$}
            \STATE Draw $M$ samples jointly from GP and its gradient at $\mathbf{x}^t$, conditional on $\mathbb{D}_n$;
            \STATE Evaluate gradients of acquisition function as in \eqref{expected gradient};
            \STATE Evaluate Objective Integral as given by \eqref{MCMCestimator};
            \STATE Evaluate Penalisation Integral as given by \eqref{MCMC Pen};
            \STATE Evaluate $\widetilde{\Phi}^*$ as given by \eqref{final Phi}:
            \STATE  $\mathbf{x}_r^{t+1} \gets \mathbf{x}_r^t +\epsilon \widetilde{\Phi}^*(\mathbf{x}_r^t)$;
        \ENDFOR
    \STATE $\mathbf{x}^{\text{new}} \gets \mathbf{x}^{T}$

    \STATE$ \mathbb{D}_{n} \gets \mathbb{D}_{n} \cup \{\mathbf{x}^{\text{new}}_{j}, f(\mathbf{x}^{\text{new}}_{j})\}_{j=1}^{N} $
\ENDFOR
\end{algorithmic}
\end{algorithm}

\begin{algorithm}
\caption{Batch BO via Wasserstein PGF}\label{algowass}
\KwInput{Dataset $\mathbb{D}_{n} = \{ \mathbf{x}_i, f(x_i)\}_{i = 1}^{_{n}}$, { batch size $q$, ensemble size $N$,} length of optimisation $T$, number of BO iterations $I$, number of GP samples $M$, regularisation constant $\alpha_w$, step size $\epsilon_w$.}
\KwOutput{Updated dataset $\mathbb{D}_{n}$ with locations at which function has been evaluated.}
\begin{algorithmic}[1]
\FOR{$i= 1,\dots,I$}
    \STATE Fit a GP to $\mathbb{D}_n$; 
        \STATE select random initial location $\mathbf{x}^0 \in \mathbb{R}^{N\times d}$;
         \FOR{$t=0,\dots,T-1$}
            \STATE Draw $M$ samples jointly from GP and its gradient at $\mathbf{x}^t$, conditional on $\mathbb{D}_n$;
            \STATE Evaluate gradients of acquisition function as in \eqref{expected gradient};
            \STATE Evaluate $\tilde{H}$ as given by \eqref{H-Ustats};
            \STATE Evaluate Euler Maruyama step as given in \eqref{EM - step} and update $\mathbf{x}^{t+1}$;
        \ENDFOR
    \STATE $\mathbf{x}^{\text{new}} \gets \mathbf{x}^{T}$

    \STATE$ \mathbb{D}_{n} \gets \mathbb{D}_{n} \cup \{\mathbf{x}^{\text{new}}_{j}, f(\mathbf{x}^{\text{new}}_{j})\}_{j=1}^{N} $
\ENDFOR
\end{algorithmic}
\end{algorithm}

\section{Numerical Results}

We demonstrate the behaviour of this scheme on a number of benchmark functions for Bayesian Optimisation, specifically the Ackley and Griewank functions, both in the 2 dimensional and higher dimensional cases. We compare our method to qEI as implemented in the Cornell-Moe package\footnote{\url{https://github.com/wujian16/Cornell-MOE}}, and Local Penalization~\cite{gonzalez2016batch}.

Both gradient flows schemes are run for a fixed number of $I$ steps and the final $N$ points are then sent for evaluations. Both methodologies additionally contain other parameters linked to the amount of regularisation we want to perform ($\alpha$, the Stein kernel and its length-scale and $\alpha_w$), the step size by which particles are propagated ($\epsilon$ and $\epsilon_w$) and the size of the batch $q$ for the target q-EI. We therefore need to calibrate and tune these values accordingly. Ideally, we want to allow for just enough regularisation so that q-EI is maximised in expectation while also retaining diversity in the particles, and set a step size so that the particles converge to an equilibrium within the time limit $I$, without being pushed too strongly. At the current moment, we set these parameters a priori by monitoring the norm of the vector that moves the particles, namely $\widetilde{\Phi}^*$ as given in \eqref{final Phi} for the Stein flow and the drift term in \eqref{EM - step} for the Wasserstein flow. More specifically, the numbers are set so that the norm of these terms decays to 0, meaning that the particles do not evolve anymore, having reached an equilibrium. The Stein kernel is fixed for all methods to be Matern $5/2$ and other parameter choices and settings are provided in Table \ref{Table Params}. \\

\begin{table}[]
\begin{tabular}{c|cc|cc|cc|cc|cc|}
\cline{2-11}
                                                                                                    & \multicolumn{2}{c|}{\textbf{Ackley}}                & \multicolumn{2}{c|}{\textbf{Ackley 5}}              & \multicolumn{2}{c|}{\textbf{Griewank}}              & \multicolumn{2}{c|}{\textbf{Griewank5}}             & \multicolumn{2}{c|}{\textbf{Lorentz63}}             \\ \cline{2-11} 
                                                                                                    & \multicolumn{1}{c|}{\textit{Stein}} & \textit{Wass} & \multicolumn{1}{c|}{\textit{Stein}} & \textit{Wass} & \multicolumn{1}{c|}{\textit{Stein}} & \textit{Wass} & \multicolumn{1}{c|}{\textit{Stein}} & \textit{Wass} & \multicolumn{1}{c|}{\textit{Stein}} & \textit{Wass} \\ \hline
\multicolumn{1}{|c|}{\textbf{Bounds}}                                                               & \multicolumn{2}{c|}{$(-5,5)^2$}                     & \multicolumn{2}{c|}{$(-3,3)^5$}                     & \multicolumn{2}{c|}{$(-500,500)^2$}                 & \multicolumn{2}{c|}{$(-500,500)^5$}                 & \multicolumn{2}{c|}{$(20,40)\times (0,10)$}         \\ \hline
\multicolumn{1}{|c|}{\textbf{\begin{tabular}[c]{@{}c@{}}Starting \\ Points\end{tabular}}}           & \multicolumn{2}{c|}{10}                             & \multicolumn{2}{c|}{50}                             & \multicolumn{2}{c|}{10}                             & \multicolumn{2}{c|}{50}                             & \multicolumn{2}{c|}{5}                              \\ \hline
\multicolumn{1}{|c|}{\textbf{N}}                                                                    & \multicolumn{2}{c|}{10}                             & \multicolumn{2}{c|}{10}                             & \multicolumn{2}{c|}{10}                             & \multicolumn{2}{c|}{10}                             & \multicolumn{2}{c|}{10}                             \\ \hline
\multicolumn{1}{|c|}{\textbf{q}}                                                                    & \multicolumn{1}{c|}{3}              & 3             & \multicolumn{1}{c|}{3}              & 3             & \multicolumn{1}{c|}{3}              & 3             & \multicolumn{1}{c|}{3}              & 3             & \multicolumn{1}{c|}{3}              & 3             \\ \hline
\multicolumn{1}{|c|}{\textbf{M}}                                                                    & \multicolumn{2}{c|}{500}                            & \multicolumn{2}{c|}{1000}                           & \multicolumn{2}{c|}{500}                            & \multicolumn{2}{c|}{1000}                           & \multicolumn{2}{c|}{500}                            \\ \hline
\multicolumn{1}{|c|}{\textbf{T}}                                                                    & \multicolumn{1}{c|}{3000}           & 1500          & \multicolumn{1}{c|}{3000}           & 1500          & \multicolumn{1}{c|}{5000}           & 2000          & \multicolumn{1}{c|}{5000}           & 3000          & \multicolumn{1}{c|}{6000}           & 3000          \\ \hline
\multicolumn{1}{|c|}{\textbf{$\alpha$}}                                                             & \multicolumn{1}{c|}{0.02}           & *             & \multicolumn{1}{c|}{0.01}           & *             & \multicolumn{1}{c|}{1e-3}           & *             & \multicolumn{1}{c|}{1e-4}           & *             & \multicolumn{1}{c|}{0.05}           & *             \\ \hline
\multicolumn{1}{|c|}{\textbf{\begin{tabular}[c]{@{}c@{}}Length-scale \\ Kernel Stein\end{tabular}}} & \multicolumn{1}{c|}{0.5}            & *             & \multicolumn{1}{c|}{0.5}            & *             & \multicolumn{1}{c|}{0.5}            & *             & \multicolumn{1}{c|}{0.5}            & *             & \multicolumn{1}{c|}{1}              & *             \\ \hline
\multicolumn{1}{|c|}{\textbf{$\epsilon$}}                                                           & \multicolumn{1}{c|}{0.5}            & *             & \multicolumn{1}{c|}{0.5}            & *             & \multicolumn{1}{c|}{0.3}            & *             & \multicolumn{1}{c|}{0.5}            & *             & \multicolumn{1}{c|}{0.5}            & *             \\ \hline
\multicolumn{1}{|c|}{\textbf{$\alpha_w$}}                                                           & \multicolumn{1}{c|}{*}              & 0.01          & \multicolumn{1}{c|}{*}              & 5e-3          & \multicolumn{1}{c|}{*}              & 1e-3          & \multicolumn{1}{c|}{*}              & 5e-4          & \multicolumn{1}{c|}{*}              & 5e-3          \\ \hline
\multicolumn{1}{|c|}{\textbf{$\epsilon_w$}}                                                         & \multicolumn{1}{c|}{*}              & 0.5           & \multicolumn{1}{c|}{*}              & 0.5           & \multicolumn{1}{c|}{*}              & 0.3           & \multicolumn{1}{c|}{*}              & 0.5           & \multicolumn{1}{c|}{*}              & 0.5           \\ \hline
\multicolumn{1}{|c|}{\textbf{\begin{tabular}[c]{@{}c@{}}GP model\\  with noise\end{tabular}}}       & \multicolumn{2}{c|}{False}                          & \multicolumn{2}{c|}{False}                          & \multicolumn{2}{c|}{True}                           & \multicolumn{2}{c|}{True}                           & \multicolumn{1}{c|}{True}           & True          \\ \hline
\end{tabular}\caption{List of all fixed parameters used for the examples.}\label{Table Params}
\end{table}

In order to compare methods, we run the optimisation routines to test robustness to different starting points. Since the methods might yield different results even with the same random initialisation, we will also test for robustness within the same experiment. Therefore, we initialise $10$ random locations, and for each of these, we run the optimisation procedure $5$ times. As a metric of comparison we use log regret, i.e. the logarithm of the minimum function value observed at each iteration, more formally: $\mathrm{Log-Regret}(\mathcal{D}_n) = \log [\min_{\mathbf{x} \in \mathcal{D}_n} f(\mathbf{x}) - \widetilde{f} ]$, where $\widetilde{f}$ indicates the true minimum value. We record the median log-regret across the 5 routines for each starting locations. The plots show the median, $10\%$ and $90\%$ quantiles of the 10 previously evaluated medians.  We demonstrate the convergence of the optimisation procedure in Figure \ref{SyntheticFunc} to highlight the evolution of the optimisation for each iteration. In addition, for a more synthetic overview of the methods, we have included a box plot in Figure \ref{SyntheticFuncBox} with the Log-Regrets at the end of the optimisation routine, to compare what the final results are.  These results indicate that the proposed methodology is at least competitive with state of the art batch-BO algorithms, but in most cases also demonstrating considerable improvement both in terms of log-regret and uncertainty.

\begin{figure}
     \centering
     \begin{subfigure}[b]{0.4\textwidth}
         \centering
         \includegraphics[width=6cm]{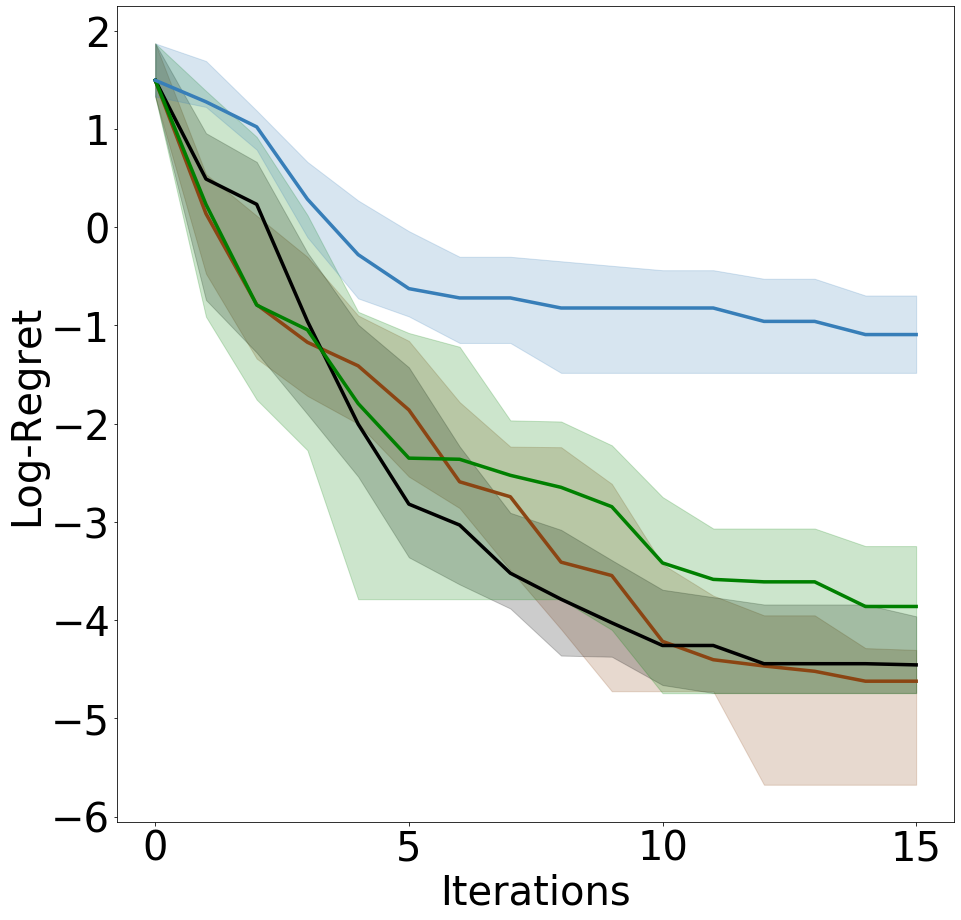}
         \caption{Ackley}
     \end{subfigure}
     \begin{subfigure}[b]{0.4\textwidth}
         \centering
         \includegraphics[width=6cm]{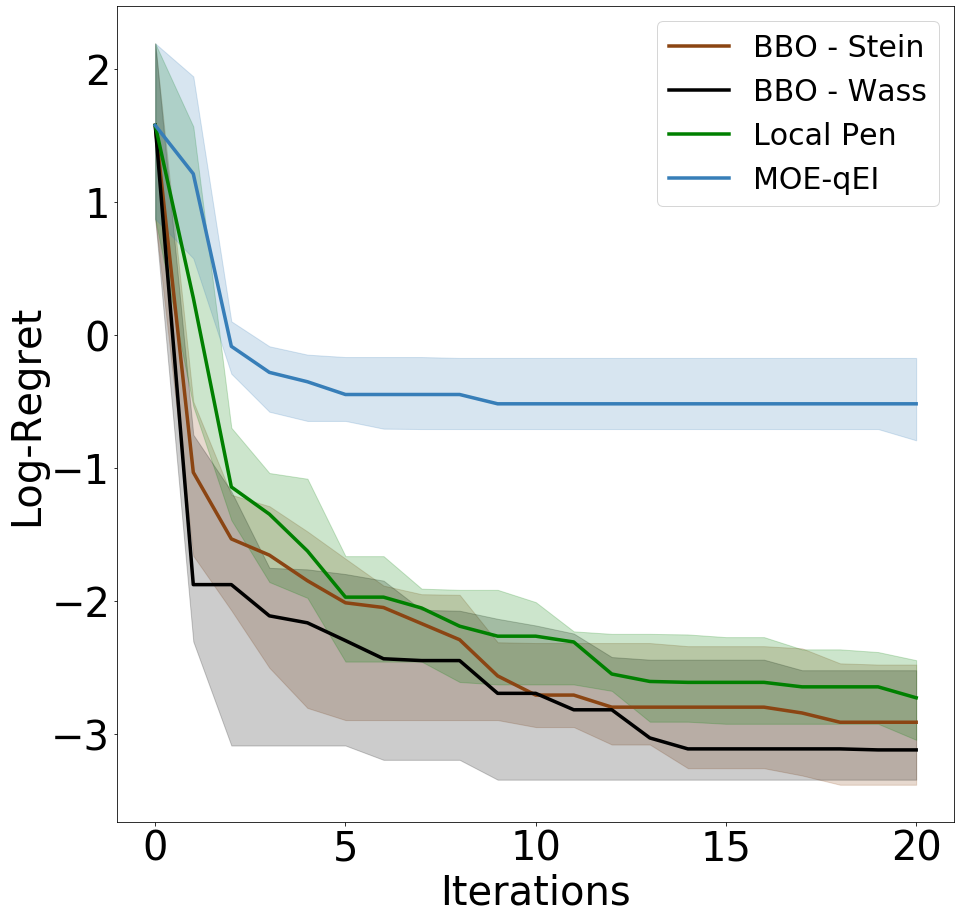}
         \caption{Griewank}
     \end{subfigure}
     \hfill
     \begin{subfigure}[b]{0.4\textwidth}
         \centering
         \includegraphics[width=6cm]{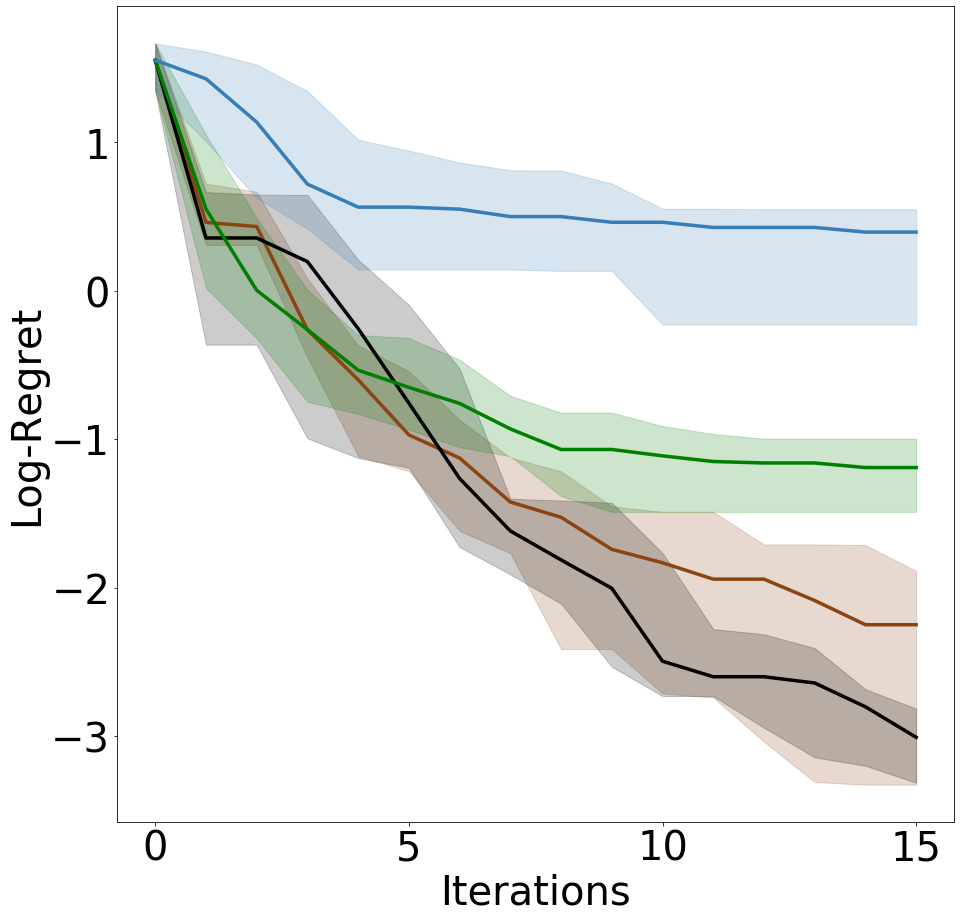}
         \caption{Ackley5}
     \end{subfigure}
     \begin{subfigure}[b]{0.4\textwidth}
         \centering
         \includegraphics[width=6cm]{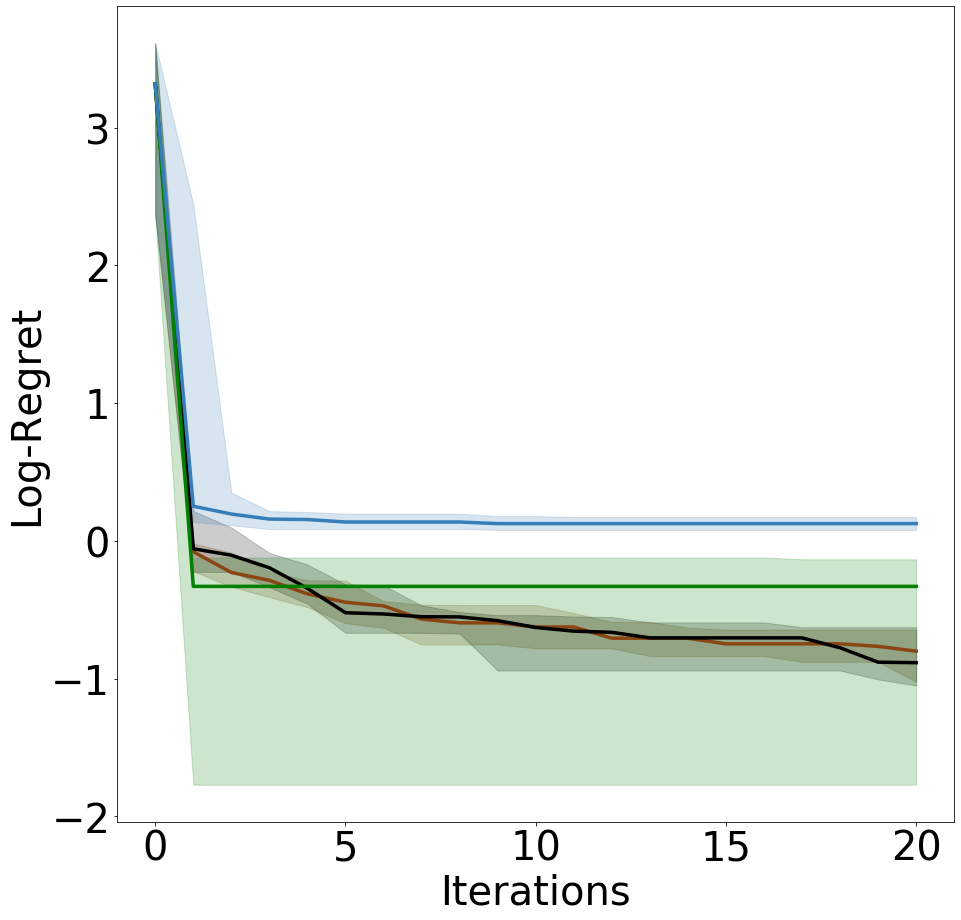}
         \caption{Griewank5}
     \end{subfigure}
        \caption{Comparison of Log-regrets for synthetic functions over the iterations of the BO routines.}\label{SyntheticFunc}
\end{figure}

\begin{figure}
     \centering
     \begin{subfigure}[b]{0.4\textwidth}
         \centering
         \includegraphics[width=6cm]{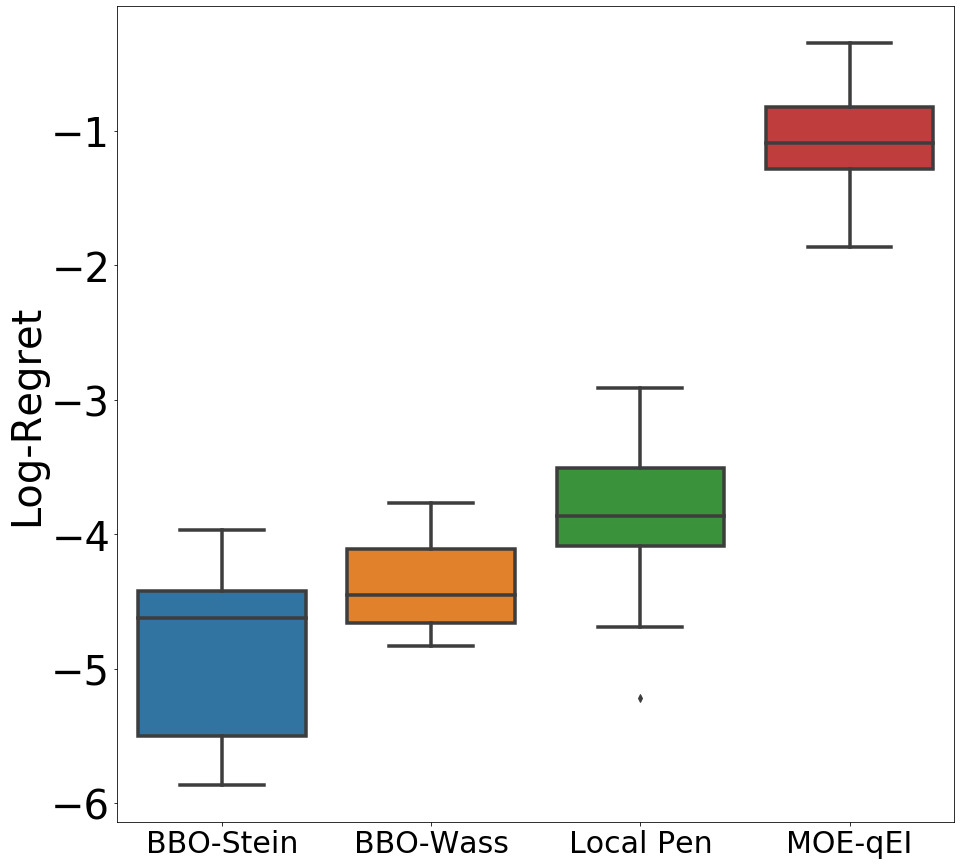}
         \caption{Ackley}
     \end{subfigure}
     \begin{subfigure}[b]{0.4\textwidth}
         \centering
         \includegraphics[width=6cm]{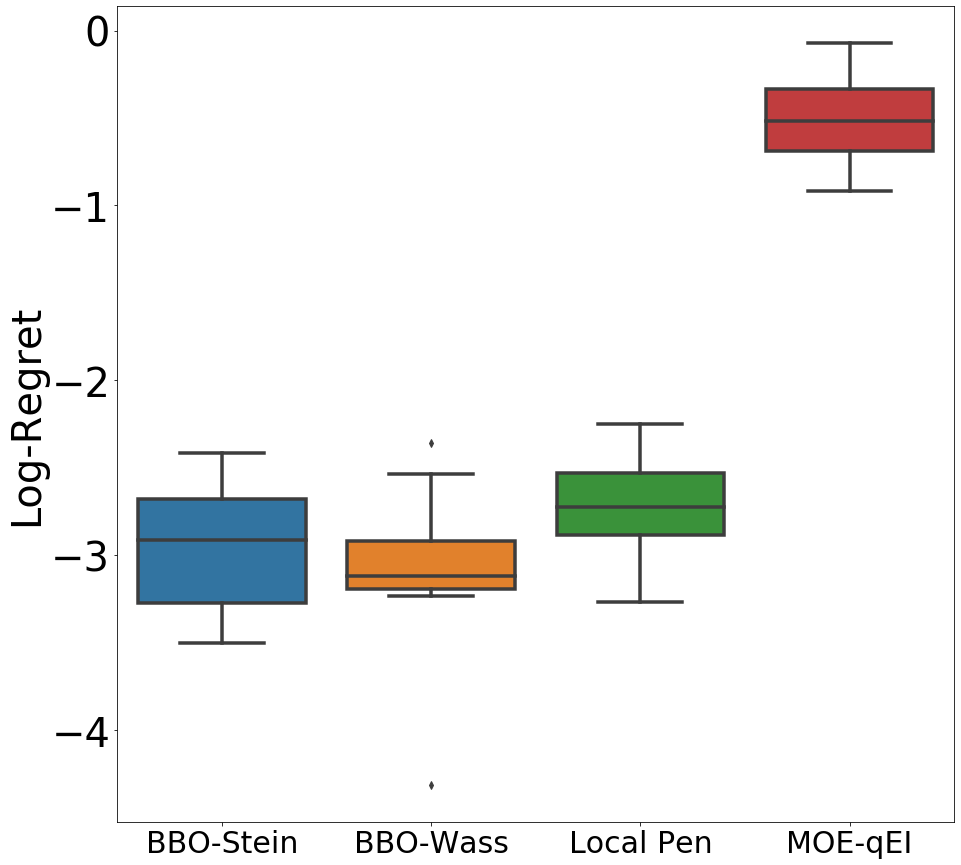}
         \caption{Griewank}
     \end{subfigure}
     \hfill
     \begin{subfigure}[b]{0.4\textwidth}
         \centering
         \includegraphics[width=6cm]{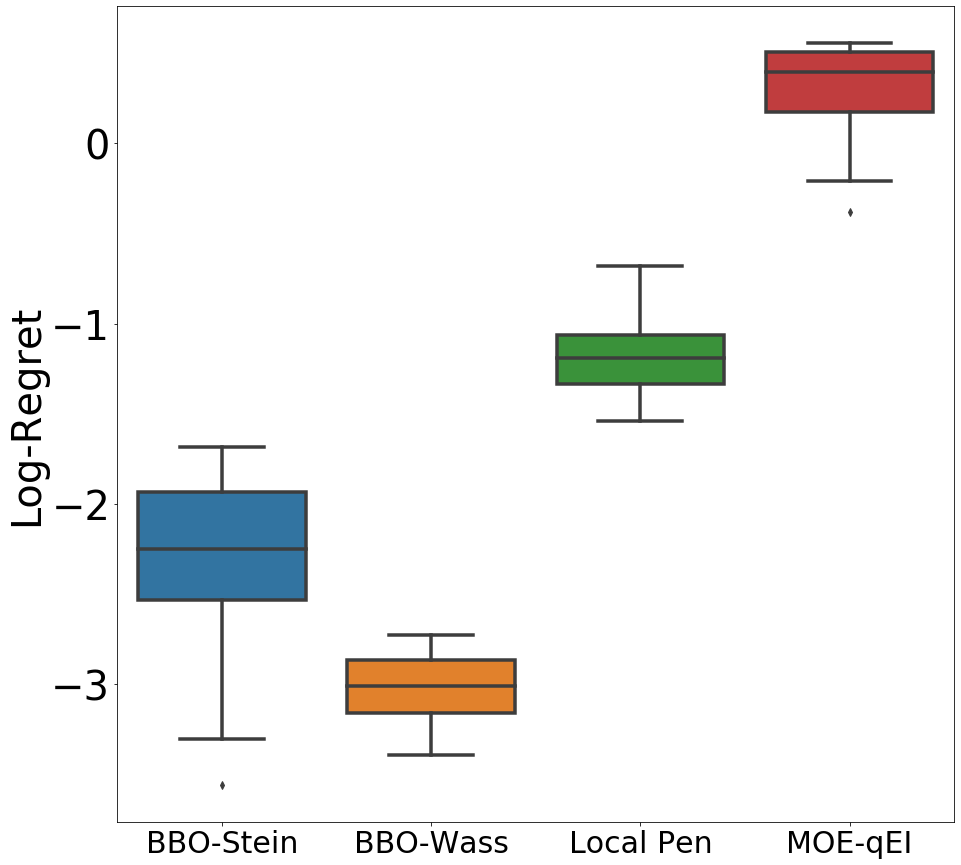}
         \caption{Ackley5}
     \end{subfigure}
     \begin{subfigure}[b]{0.4\textwidth}
         \centering
         \includegraphics[width=6cm]{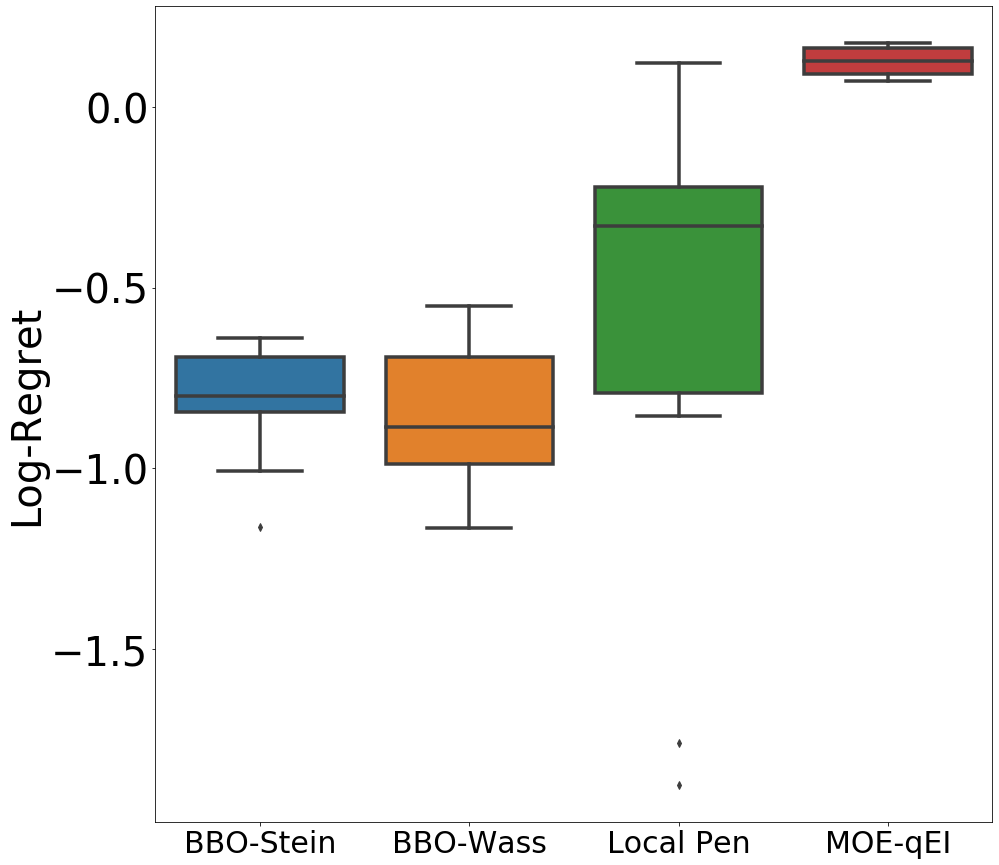}
         \caption{Griewank5}
     \end{subfigure}
    \caption{Box Plot with Log-regrets for synthetic functions at the end of the BO routines.}\label{SyntheticFuncBox}
\end{figure}

\subsection{Calibration of the Lorentz' 63 Oscillator from Time-Averaged Data}
As a further example we will use the proposed batch BO method to calibrate a nonlinear system of ordinary differential equations to data.  More specifically, we consider the problem of calibrating the parameters of the Lorentz' 63 model based on time-averaged measurements of the trajectory, similar to  \cite{duncan2021ensemble}.  The 3-dimensional Lorenz equations can be written as: 
\begin{equation} 
\label{eq:lorenz}\begin{cases} 
      \dot{x}_1 = 10(x_2-x_1), & \nonumber\\
      \dot{x}_2 = rx_1 -x_2 - x_1x_3, & \\
      \dot{x}_3 = x_1x_2 -bx_3, \nonumber
   \end{cases}
\end{equation}
with parameters $\theta = (r,b) \in \mathbb{R}^2_+$. In this example, we aim to infer the value of $\theta$ from time-averages of a function of the solution $(x_1, x_2, x_3)$ over a time-interval of duration $L$. More specifically, we assume that the data $y$ satisfies
\begin{equation}
{y} = G(\hat{\theta}) + \xi_{obs}, \quad \xi_{obs} \sim \mathcal{N}(0,\Delta_{obs}),
\end{equation}
where
\begin{align}\label{G}
    G(\theta) & = \frac{1}{L} \int_{0}^L \phi(\mathbf{u}(s,\theta)) ds ,
\end{align}
for a function $\phi$ defined on the state space.  We note that $G(\theta)$ depends on initial condition $u_0$, which we view as a random variable distributed according to the invariant measure
of the dynamics \eqref{eq:lorenz}.   Rather than follow the approach of introducing a latent variable for $u_0$ to capture its effect on the dynamics, we shall introduce an approximation to the likelihood, building on the synthetic likelihood approach of \cite{wood2010statistical}.  Assuming that the central limit holds for these dynamics we are justified in the approximation
\begin{align*}\label{G rewrite}
    G(\theta) &= G_0(\theta) + \mathcal{N}(0, L^{-1}\Delta(\theta)).
\end{align*}
Approximating $L^{-1}\Delta(\theta)$ by a constant covariance $\Delta_{model}$, estimated  from a single long run of the (assumed ergodic and mixing) model at a fixed parameter $\theta^\dagger$ and batched into windows of length $L$ and assuming that the observation noise is independent of the initial condition $u_0$ then we can rewrite the inverse problem as
\begin{equation}
\label{eq:recast_ivp}
 y = G_0(\theta) + \xi,
\end{equation}
where $\xi \sim \mathcal{N}(0, \Delta_{obs} + \Delta_{model})$, where $G_0(\theta)$ is the infinite time-average, i.e. taking the limit $L\rightarrow \infty$ in \eqref{G}.  Our
goal is to solve the inverse problem \eqref{eq:recast_ivp}, assuming we only have access to noisy, biased evaluations of $G_0(\theta)$ through $G(\theta)$.   

To generate a sample of the data, we evaluate \eqref{G} at a random initial condition drawn from an approximation o the invariant measure and then
add a sample of $\xi_{obs}$.  In this example, we shall fix  $L=10$ and choose the true (unobserved) parameters to be ${\theta}^\dagger = (28, 8/3)$. We choose $\phi: \mathbb{R}^3 \to \mathbb{R}^9$ given by:
\begin{equation}
\phi(x_1,x_2,x_3) = (x_1, x_2, x_3, x_1^2, x_2^2, x^2_3, x_1x_2, x_2x_3, x_1x_3),\end{equation}
so that the forward map $G$ captures the first and second time-averaged moments of the solution.

The parameters ${\theta}$ will be inferred based only on the  $G(\cdot)$. In our example we will let $\Delta_{obs} = 0$, meaning that randomness comes only from the unknown initial condition.
The data-misfit function is then given by:
\begin{equation}\label{likelihood}
    M(\theta) := \left[\frac{1}{2} \langle (\hat{y} - G(\theta) , \Delta_{model}^{-1} (\hat{y} - G(\theta)  \rangle\right]^t.
\end{equation}
In order to ease the computation to find the minimum, we have introduced a tempering constant $t$, which in this example is set to 0.3. 

The unique challenge in this problem is that the misfit function poses several difficulties to standard optimisation methods, due to the  roughness arising from the noisy evaluations. This feature is clearly shown in Figure \ref{Image Likelihood}, where the misfit function $M$ is evaluated with $r$ fixed to the true value and $b$ is allowed to vary. In addition, as each evaluation of the objective function requires multiple ODE solves to compute \eqref{G}, this is considerably expensive, particularly when the time horizon $L$ is large. To calibrate the model we perform batch Bayesian optimisation where, for each iteration, the misfit function is approximated using the following steps:

\begin{enumerate}
    \item Select new batch of parameter points $\boldsymbol{\theta} = (\theta_1, \ldots, \theta_N) \in \mathbb{R}^{N \times 2}$,
    \item Approximate the solution $\mathbf{u}(l,\theta_i)$ and run a simulation for $i = 1,\ldots,N$,
    \item Evaluate $G(\theta_i)$ as given by \eqref{G} and obtain solutions $y_i$, $i = 1,\ldots,N$ through time-averaging.
    \item For $y_1, \ldots, y_N$, evaluate the misfit function  as given by \eqref{likelihood}
\end{enumerate}

\begin{figure}[htp]
    \centering
    \includegraphics[width=9cm]{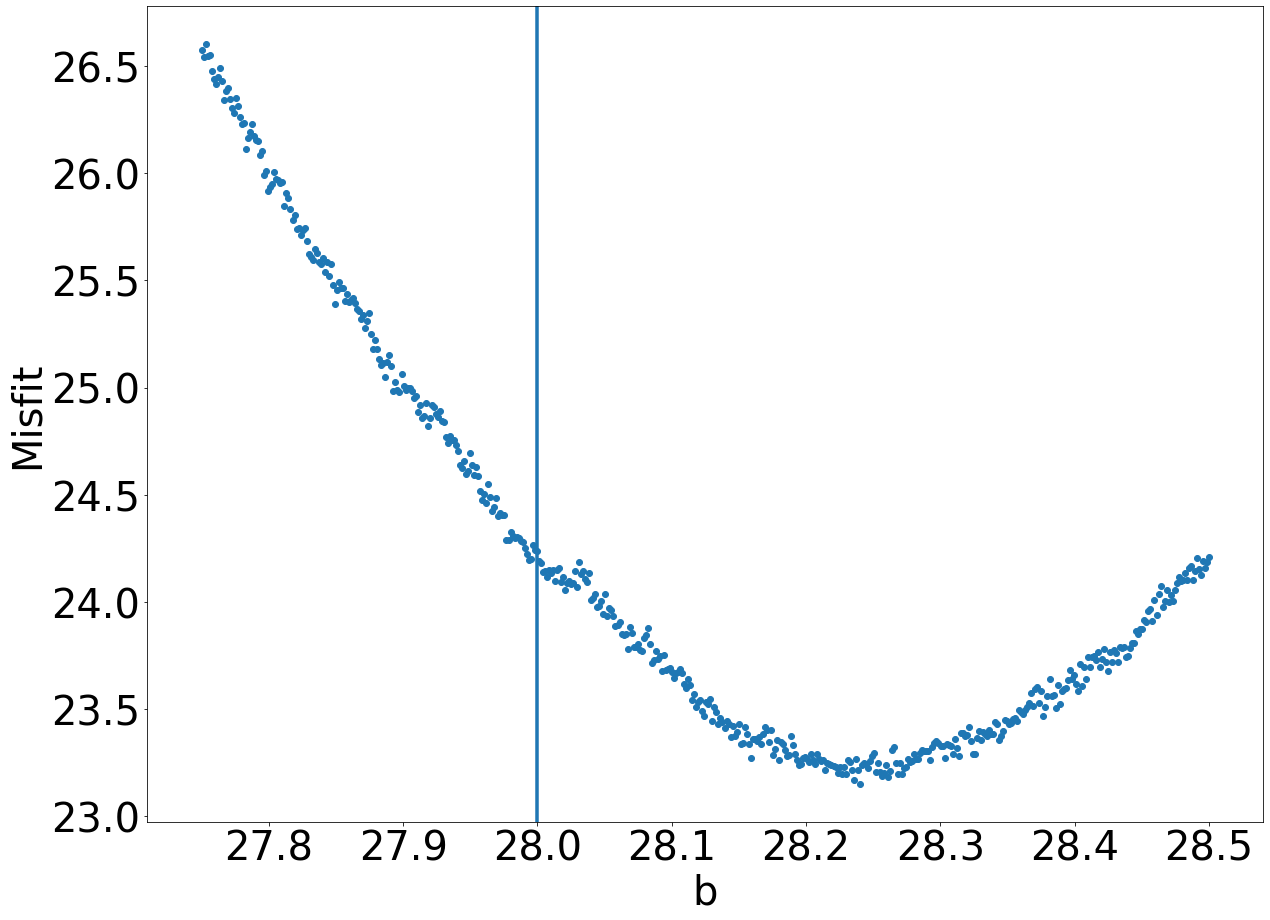}
    \caption{Data-misfit function $M$ with fixed $r= 8/3$. We note that the target function which we aim to optimise is not smooth and presents noise. Vertical line is the true value of $b$.}
    \label{Image Likelihood}
\end{figure}
We show the results of the optimisation procedure in Figure \ref{LL}. The plot shows the logarithm of the lowest misfit value observed at each iteration. Once again, we run the optimisation procedure over 10 different random initialisations, and for each we run the routines 5 times and plot the median, $10\%$ and $90\%$ quantiles of the results. More details can be found in Table \ref{Table Params}. Despite the noise in the data-misfit function, the methodologies presented in our work seem to capture the area where the true parameter lie. Lastly, our algorithms perform equivalently as Local Penalisation while outperforming MOE-qEI.

\begin{figure}[htp]
    \centering
    \includegraphics[width=9cm]{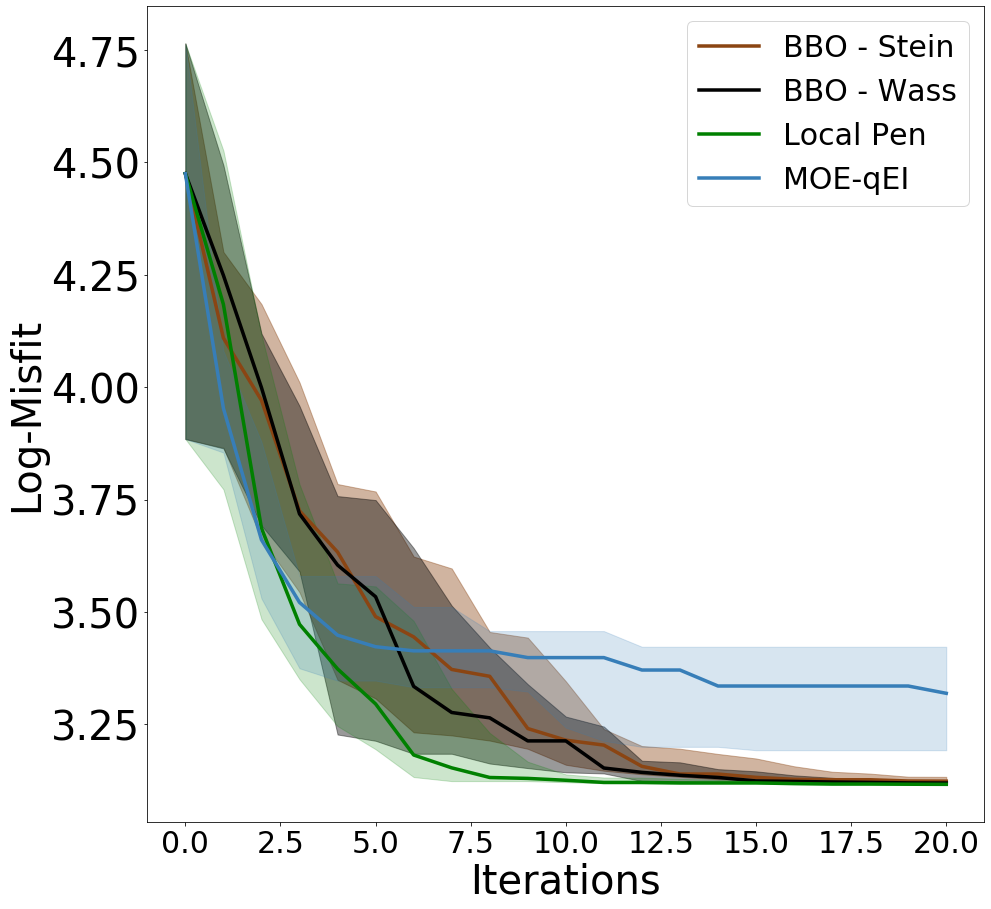}
    \caption{Comparison of Log-Misfit for Lorentz Oscillator over the iterations of the BO routines.}
    \label{LL}
\end{figure}

\section{Conclusion and Future Work}
In this work we proposed a reformulation of Batch Bayesian Optimisation in terms of an optimisation of an acquisition functional over the set of probability measures.  The acquisition functional  generalises the multipoint expected improvement acquisition function, similarly aiming to balance between exploitation and exploration, but also introduces a regularisation term which renders the optimisation problem strictly concave.    We demonstrate how different gradient flows of this objective will yield practical schemes for solving the inner-loop optimisation problem which is typically the main challenge in existing batch BO schemes, and consider two specific cases, namely one based on a Stein geometry and another based on Wasserstein.   The strength of the resulting scheme is that it provides rapid batch exploration of the maxima of the objective function while promoting diversity between the points.  Unlike other batch BO methods which introduce hard constraints on the distance between points, in the proposed schemes diversity is promoted naturally through a regularisation term which automatically adapts to the geometry of the partially observed objective function.     

There are various avenues of further investigation within this framework.  Firstly, understanding the objective landscape introduced by the concave relaxation is crucial, particularly in terms of characterising the extrema of this new acquisition function,  and how it influences the exploration-exploitation strategy.  The resulting schemes are clearly sensitive to hyper-parameter value, including step-sizes, kernel-bandwidths and the regularisation strength parameter.  Exploring means of adapting these values dynamically, for example following a strategy similar to that which was proposed in \cite{duncan2021ensemble} would be a natural candidate for future work.   Finally, while the schemes perform demonstrably well on low-to-middle dimensional problems, as with all Gaussian process based schemes, they do not scale well with dimension.  Investigating kernels which are able to mitigate the curse of dimensionality would enable the methodology to scale better.  This could include approaches based on Random Fourier Features~\cite{rahimi2007random} or alternatively constructing kernels based on low-rank projections similar to  \cite{liu2022grassmann}.

\section*{Acknowledgements}
This work has been carried out within the framework of the EUROfusion Consortium, funded by the European Union via the Euratom Research and Training Programme (Grant Agreement No 101052200 - EUROfusion). Views and opinions expressed are however those of the author(s) only and do not necessarily reflect those of the European Union or the European Commission. Neither the European Union nor the European Commission can be held responsible for them. This work was part funded by the RCUK Energy Programme [grant number EP/P012450/1]. 
AD was supported by Wave 1 of The UKRI Strategic Priorities Fund under the EPSRC Grant EP/T001569/1 and EPSRC Grant EP/W006022/1, particularly the ``Ecosystems of Digital Twins'' theme within those grants \& The Alan Turing Institute.
SC was supported by the Alan Turing Institute.  K.C.Z. was supported by the Leverhulme Trust grant 2020-310 and EPSRC grant EP/V006177/1.
\bibliographystyle{unsrt}  
\bibliography{references}

\appendix
\section{Full derivation of BBO via Stein Gradient Flow}\label{sec:AppA}
Full derivation of Equations \eqref{ObjectiveF}, \eqref{ObjectiveReg} and \eqref{Phi}. Recall that our objective is to evaluate:
\begin{equation}\label{ObjectivePhi_appendix}
    \Phi^* = \argmax_\Phi \frac{d}{d\epsilon} \left[ F[\mathbf{T}_\#(\mu)] + \alpha \mathrm{Reg}[\mathbf{T}_\#(\mu)] \right]  \Bigr \rvert_{\epsilon = 0}.
\end{equation}

Let us denote $\mathbf{x}_i^\prime = \mathbf{T} (\mathbf{x}_i)$. By using the definition of pushforward measure:
\begin{align*}
    F\left[\mathbf{T}_\#(\mu)\right] &= \mathbb{E}_{\mathbf{X}^\prime_1,\ldots,\mathbf{X}^\prime_q \sim \mathbf{T}_\#(\mu)}\left[ \mathrm{EI}(\mathbf{X}^\prime_1,\ldots,\mathbf{X}^\prime_q) \right], \nonumber \\
    & = \int \ldots \int \mathrm{EI}(\mathbf{x}^\prime_1,\ldots,\mathbf{x}^\prime_q)\mathbf{T}_\# (\mu)(d\mathbf{x}^\prime_1)\ldots\mathbf{T}_\# (\mu)(d\mathbf{x}^\prime_q), \nonumber \\
    & = \int \ldots \int \mathrm{EI}(\mathbf{x}_1 + \epsilon \Phi(\mathbf{x}_1),\ldots,\mathbf{x}_q + \epsilon \Phi(\mathbf{x}_q)) \mu(d\mathbf{x}_1)\ldots\mu(d\mathbf{x}_q). \nonumber
\end{align*}

Therefore we can take the derivative and we obtain:
\begin{align*}
    \frac{d}{d\epsilon} F\left[\mathbf{T}_\#(\mu)\right]\Bigr \rvert_{\epsilon = 0} &= \frac{d}{d\epsilon} \left[ \int \ldots \int \mathrm{EI}(\mathbf{x}_1 + \epsilon \Phi(\mathbf{x}_1),\ldots,\mathbf{x}_q + \epsilon \Phi(\mathbf{x}_q)) \mu(d\mathbf{x}_1)\ldots\mu(d\mathbf{x}_q) \right] \Bigr \rvert_{\epsilon = 0}, \nonumber \\
    & = \int \ldots \int \frac{d}{d\epsilon} \left[ \mathrm{EI}(\mathbf{x}_1 + \epsilon \Phi(\mathbf{x}_1),\ldots,\mathbf{x}_q + \epsilon \Phi(\mathbf{x}_q)) \right] \Bigr \rvert_{\epsilon = 0} \mu(d\mathbf{x}_1)\ldots\mu(d\mathbf{x}_q), \\
    & = \int \ldots \int \sum_{i=1}^q \frac{\delta \mathrm{EI}}{\delta \mathbf{x}_i }(\mathbf{x}_1,\ldots,\mathbf{x}_q)\Phi(\mathbf{x}_i) \mu(d\mathbf{x}_1)\ldots\mu(d\mathbf{x}_q).
\end{align*}

Now let $\mathbf{x}^\prime = \mathbf{T} (\mathbf{x})$ where $\mathbf{x}^\prime, \mathbf{x} \in \mathbb{R}^d$. Hence, $d\mathbf{x}^\prime = |\det(\mathbf{J T})(\mathbf{x})| (d\mathbf{x})$ where $\mathbf{J T}$ indicated the Jacobian matrix of $\mathbf{T}$. Then we can write:
\begin{align*}
    \mathrm{Reg}\left[\mathbf{T}_\#(\mu)\right] &= -\int  \log(\mathbf{T}_\#(\mu)(\mathbf{x}^\prime) ) \mathbf{T}_\#(\mu)(d\mathbf{x}^\prime),\\
    &= - \int \log \left[ \frac{\mu(\mathbf{x})}{|\det(\mathbf{J T})(\mathbf{x})|} \right] \mu(d\mathbf{x}), \\  
    &= -\int \log [{\mu(\mathbf{x})}] \mu(d\mathbf{x}) +\int \log[{|\det(\mathbf{J T})(\mathbf{x})|}] \mu(d\mathbf{x}).  
\end{align*}

Now, differentiating with respect to $\epsilon$:
\begin{align*}
    \frac{d}{d\epsilon}\mathrm{Reg}\left[\mathbf{T}_\#(\mu)\right]\Bigr \rvert_{\epsilon = 0}
    &=\frac{d}{d\epsilon} \left[ -\int \log [{\mu(\mathbf{x})}] \mu(d\mathbf{x}) +\int \log[{|\det(\mathbf{J T})(\mathbf{x})|}] \mu(d\mathbf{x})  \right]\Bigr \rvert_{\epsilon = 0}, \\ 
    &= \frac{d}{d\epsilon} \left[ \int  \log[{|\det(\mathbf{J T})(\mathbf{x})|}] \mu(d\mathbf{x})  \right]\Bigr \rvert_{\epsilon = 0}, \\
    &=  \int \frac{d}{d\epsilon} \left[ \log[{|\det(\mathbf{J T})(\mathbf{x})|}]\right]\Bigr \rvert_{\epsilon = 0} \mu(d\mathbf{x}), \\
    &= \int \mathrm{trace} \left[ (\mathrm{Id} + \epsilon\mathbf{J} \Phi(\mathbf{x}))^{-1}\Bigr \rvert_{\epsilon = 0} \frac{d}{d\epsilon}(\mathrm{Id} + \epsilon\mathbf{J}\Phi(\mathbf{x}))\Bigr \rvert_{\epsilon = 0}    \right] \mu(d\mathbf{x}), \\
    &= \int \mathrm{trace} \left[ \mathbf{J}\Phi(\mathbf{x}) \right] \mu(d\mathbf{x}), \\
    &= \int \nabla \cdot \Phi(\mathbf{x}) \mu(d\mathbf{x}).
\end{align*}
 
Equation \eqref{ObjectivePhi_appendix} can then be rewritten as:
\begin{equation*}
    \Phi^* = \argmax_\Phi \int \ldots \int \sum_{i=1}^q \frac{\delta \mathrm{EI}}{\delta \mathbf{x}_i }(\mathbf{x}_1,\ldots,\mathbf{x}_q)\Phi(\mathbf{x}_i) \mu(d\mathbf{x}_1)\ldots\mu(d\mathbf{x}_q) + \int \nabla \cdot \Phi(\mathbf{x}) \mu(d\mathbf{x})
\end{equation*}

\section{Full Derivation of BBO via Wasserstein Gradient Flow}

Full derivation of Equations \eqref{Wasserstein GF}, \eqref{FV-OBJ} and \eqref{FV-REG} and \eqref{Flow}. Recall that our objective is calculating the first variation of $L[\mu_t]$ so that we can specify the gradient flow given by: 
\begin{equation}\label{Wasserstein GF - Appendix}
    \frac{\partial \mu_t}{\partial t} = \mathrm{div}(\mu_t \nabla_\mathbf{x} L'[\mu_t]),
\end{equation}
where $L'[\mu_t]$ is the first variation of $L$. We can split the calculations focusing firstly on $F$ and then the regularisation term. This can be calculated using the Gateaux derivative and we can write, for $\epsilon > 0$ and $\nu$ such that $\mu + \epsilon \nu \in \mathcal{P}(\mathcal{D})$: 
\begin{align}\label{FV-OBJ - Appendix}
    \frac{d}{d\epsilon} F[\mu + \epsilon \nu]  \Bigr \rvert_{\epsilon = 0} &  =
    \frac{d}{d\epsilon} \left[ \int \ldots \int \text{q-EI}(\mathbf{x_1}, \ldots, \mathbf{x_q}) (\mu+ \epsilon \nu)(d\mathbf{x}_1)\ldots(\mu+ \epsilon \nu)(d\mathbf{x}_q)\right] \Bigr \rvert_{\epsilon = 0},  \nonumber \\
    &=  \frac{d}{d\epsilon} \left[ \int \ldots \int \text{q-EI}(\mathbf{x_1}, \ldots, \mathbf{x_q}) (\mu(\mathbf{x}_1)+ \epsilon \nu(\mathbf{x}_1))\ldots(\mu(\mathbf{x}_q)+ \epsilon \nu(\mathbf{x}_q))d\mathbf{x}_1\ldots d\mathbf{x}_q \right] \Bigr \rvert_{\epsilon = 0}, \nonumber \\
    & = \frac{d}{d\epsilon} \left[ \int \ldots \int \text{q-EI}(\mathbf{x_1}, \ldots, \mathbf{x_q}) \mu(\mathbf{x}_1)\ldots\mu(\mathbf{x}_q d\mathbf{x}_1\ldots d\mathbf{x}_q \right] \Bigr \rvert_{\epsilon = 0} + \nonumber \\
    & +\frac{d}{d\epsilon} \left[ \epsilon \int \ldots \int \text{q-EI}(\mathbf{x_1}, \ldots, \mathbf{x_q}) \sum_{i=1}^q \nu(\mathbf{x}_i) \prod_{j \neq i} \mu(\mathbf{x}_i) d\mathbf{x}_1\ldots d\mathbf{x_q}  \epsilon \right] \Bigr \rvert_{\epsilon = 0} + \nonumber \\
    & + \frac{d}{d\epsilon} \left[ \epsilon^2(\ldots) + \ldots + \epsilon^q(\ldots) \right] \Bigr \rvert_{\epsilon = 0}, \nonumber \\
    & = \int \ldots \int \text{q-EI}(\mathbf{x_1}, \ldots, \mathbf{x_q}) \sum_{i=1}^q \nu(\mathbf{x}_i) \prod_{j \neq i} \mu(\mathbf{x}_i) d\mathbf{x}_1\ldots d\mathbf{x_q}.
\end{align}

Since q-EI is invariant under permutations, we can simplify the above integral to obtain: 
\begin{align}\label{FV-OBJ-FINAL - Appendix}
    \frac{d}{d\epsilon} F[\mu + \epsilon \nu]  \Bigr \rvert_{\epsilon = 0}  & = q \int \ldots \int \text{q-EI}(\mathbf{x_1}, \ldots, \mathbf{x_q})\mu(d\mathbf{x}_1)\ldots\mu(\mathbf{x}_{q-1})\nu(d\mathbf{x}_q), \nonumber \\
    & =  \int \left( q\int \ldots \int \text{q-EI}(\mathbf{x_1}, \ldots, \mathbf{x_q})\mu(d\mathbf{x}_1)\ldots\mu(\mathbf{x}_{q-1}) \right) \nu(d\mathbf{x}_q).
\end{align}

Hence we have found that the first variation $F'[\mu_t] = \left( q\int \ldots \int \text{q-EI}(\mathbf{x_1}, \ldots, \mathbf{x_q})\mu(d\mathbf{x}_1)\ldots\mu(\mathbf{x}_{q-1}) \right)$. 

We need to take the gradient of this quantity. Hence we can write:
\begin{align*}
    \nabla_\mathbf{x} F'[\mu_t] = \left[ \partial_{\mathbf{x}_i} \frac{\partial F}{\partial \mu} \right]_{i=1}^q = \begin{cases} 
        0 & i\neq q,\\
        q \frac{\partial}{\partial{\mathbf{x}_q}} \int \ldots \int  \left(\text{q-EI}(\mathbf{x_1}, \ldots, \mathbf{x_q}) \right) \mu(d\mathbf{x}_1)\ldots\mu(\mathbf{x}_{q-1})  & i = q .
   \end{cases} 
\end{align*}

Finally, we are able to calculate the right hand term in equation \eqref{Wasserstein GF - Appendix} for the target functional $F$. We can write: 
\begin{align}\label{Flow F - Appendix}
    \mathrm{div}(\mu_t \nabla_{\mathbf{x}}F'[\mu_t])&  = q \frac{\partial \mu_t }{\partial {\mathbf{x}_q}} \frac{\partial}{\partial{\mathbf{x}_q}} \int \ldots \int  \left(\text{q-EI}(\mathbf{x_1}, \ldots, \mathbf{x_q}) \right) \mu(d\mathbf{x}_1)\ldots\mu(\mathbf{x}_{q-1})  + \nonumber \\
    & + q \mu_t  \frac{\partial^2}{\partial {\mathbf{x}^2_q}} \int \ldots \int \left(\text{q-EI}(\mathbf{x_1}, \ldots, \mathbf{x_q}) \right) \mu(d\mathbf{x}_1)\ldots\mu(\mathbf{x}_{q-1}).
\end{align}

We can now focus on the regularisation term.
\begin{align}\label{FV-REG - Appendix}
    \frac{d}{d\epsilon} \text{Reg}[\mu + \epsilon \nu]  \Bigr \rvert_{\epsilon = 0} & =\frac{d}{d\epsilon}  \left[ \int \log(\mu(\mathbf{x}) + \epsilon \nu(\mathbf{x})) (\mu(\mathbf{x}) + \epsilon \nu(\mathbf{x})) d\mathbf{x} \right]\Bigr \rvert_{\epsilon = 0}, \\
    &= \int \nu(\log (\mu(\mathbf{x})) + 1 )d\mathbf{x}.\nonumber
\end{align}

We derive that the first variation for the regularisation term is Reg'$[\mu] = \log (\mu) + 1 $. Hence, we can write:
\begin{align}
    \mathrm{div}(\mu_t \nabla \mathrm{Reg}'[\mu]) &= \mathrm{div}(\mu_t \nabla[ \log (\mu_t) + 1]), \nonumber \\
    & =  \mathrm{div}(\mu_t \frac{1}{\mu_t} \nabla \mu_t),\nonumber  \\
   & = \Delta \mu_t.\label{RegWass - Appendix}
\end{align}

Hence, we can combine \eqref{Flow F - Appendix} and \eqref{RegWass - Appendix} to obtain the flow.
\begin{equation}
    \frac{\partial \mu_t}{\partial t} = \partial_{\mathbf{x}}(\mu_t) H[\mu_t, \mathbf{x}] + \mu_t \partial_{\mathbf{x}}H[\mu_t, \mathbf{x}] + \Delta \mu_t, 
\end{equation}
where: 
\begin{equation}H[\mu, \mathbf{x}] = q \frac{\partial}{\partial{\mathbf{x}_q}} \int \ldots \int  \left(\text{q-EI}(\mathbf{x_1}, \ldots, \mathbf{x}) \right) \mu(d\mathbf{x}_1)\ldots\mu(\mathbf{x}_{q-1}).\end{equation}

\end{document}